\documentclass{bmvc2k} 
\usepackage{bmvc2k_natbib, times}

\newcommand{\ARXIV}[2]{#1} 

\usepackage{hyperref}

\usepackage[utf8]{inputenc} 
\usepackage[T1]{fontenc}    
\usepackage{xurl}            
\usepackage{booktabs}       

\usepackage{nicefrac}       
\usepackage{listings}
\usepackage{amsmath}
\usepackage{amssymb}
\usepackage{amsthm}
\usepackage{amsfonts}       
\usepackage{inputenc, caption}
\usepackage{balance}
\usepackage[normalem]{ulem}
\usepackage{multicol}
\usepackage{multirow}
\usepackage{placeins}
\usepackage{subfigure}
\usepackage{style}
\usepackage{tikz}
\usepackage{enumitem}

\usepackage{wrapfig}
\usepackage{tikz}

\usepackage{booktabs}\newcommand{\ra}[1]{\renewcommand{\arraystretch}{#1}}

\usetikzlibrary{shapes}
\setlist{topsep=2pt,itemsep=0.5pt}
\newcommand{\La}{\mathcal{L}}
\newcommand{\Ldata}{\mathcal{L}_\text{data}}
\newcommand{\Lrank}{\mathcal{L}_\text{rank}}

\newcommand{\Ldisto}{\mathcal{L}_\text{disto}}

\newcommand{\cN}{\mathcal{N}}
\newcommand{\cK}{\mathcal{K}}
\newcommand{\cX}{\mathcal{X}}
\newcommand{\cT}{\mathcal{T}}
\newcommand{\bR}{\mathbb{R}}

\newcommand{\killpunct}[1]{}
\newtheorem{prop}{Proposition}

\colorlet{punct}{red!60!black}
\definecolor{background}{HTML}{EEEEEE}
\definecolor{delim}{RGB}{20,105,176}
\colorlet{numb}{magenta!60!black}

\definecolor{cellgreen}{RGB}{179, 255, 179}
\definecolor{cellred}{RGB}{255, 179, 179}

\lstdefinelanguage{json}{
    basicstyle=\scriptsize\ttfamily,
    numberstyle=\scriptsize,
    stepnumber=1,
    numbersep=8pt,
    showstringspaces=false,
    breaklines=true,
    frame=lines,
    literate=
     *{0}{{{\color{numb}0}}}{1}
      {1}{{{\color{numb}1}}}{1}
      {2}{{{\color{numb}2}}}{1}
      {3}{{{\color{numb}3}}}{1}
      {4}{{{\color{numb}4}}}{1}
      {5}{{{\color{numb}5}}}{1}
      {6}{{{\color{numb}6}}}{1}
      {7}{{{\color{numb}7}}}{1}
      {8}{{{\color{numb}8}}}{1}
      {9}{{{\color{numb}9}}}{1}
      {:}{{{\color{punct}{:}}}}{1}
      {,}{{{\color{punct}{,}}}}{1}
      {\{}{{{\color{delim}{\{}}}}{1}
      {\}}{{{\color{delim}{\}}}}}{1}
      {[}{{{\color{delim}{[}}}}{1}
      {]}{{{\color{delim}{]}}}}{1},
}

\makeatletter
\newcommand\RSloop{\@ifnextchar\bgroup\RSloopa\RSloopb}
\makeatother
\newcommand\RSloopa[1]{\bgroup\RSloop#1\relax\egroup\RSloop}
\newcommand\RSloopb[1]%
  {\ifx\relax#1%
   \else
     \ifcsname RS:#1\endcsname
       \csname RS:#1\endcsname
     \else
       \GenericError{(RS)}{RS Error: operator #1 undefined}{}{}%
     \fi
   \expandafter\RSloop
   \fi
  }
\newcommand\X{0}
\newcommand\RS[1]%
  {\begin{tikzpicture}
     [every node/.style=
       {circle,draw,fill,minimum size=1.5pt,inner sep=0pt,outer sep=0pt},
      line cap=round
     ]
   \coordinate(\X) at (0,0);
   \RSloop{#1}\relax
   \end{tikzpicture}
  }
\newcommand\RSdef[1]{\expandafter\def\csname RS:#1\endcsname}
\newlength\RSu
\RSdef{i}{\draw (\X) -- +(90:\RSu) node{};}
\RSdef{l}{\draw (\X) -- +(135:\RSu) node{};}
\RSdef{r}{\draw (\X) -- +(45:\RSu) node{};}
\RSdef{I}{\draw (\X) -- +(90:\RSu) coordinate(\X I);\edef\X{\X I}}
\RSdef{L}{\draw (\X) -- +(135:\RSu) coordinate(\X L);\edef\X{\X L}}
\RSdef{R}{\draw (\X) -- +(45:\RSu) coordinate(\X R);\edef\X{\X R}}

\title{Leveraging Class Hierarchies\\ with Metric-Guided Prototype Learning}

\addauthor{Vivien Sainte Fare Garnot}{vivien.sainte-fare-garnot@ign.fr}{1}
\addauthor{Loic Landrieu}{loic.landrieu@ign.fr}{1}

\addinstitution{
 LASTIG, ENSG, IGN\\
Univ Gustave Eiffel\\
F-94160 Saint-Mande, France \\
}

\runninghead{Sainte Fare Garnot, Landrieu}{Metric-Guided Prototype Learning}

\begin{document}

\maketitle

\begin{abstract}
In many classification tasks, the set of target classes can be organized into a hierarchy.
This structure induces a semantic distance between classes, and can be summarized under the form of a \emph{cost matrix}, which defines a finite metric on the class set.
In this paper, we propose to model the hierarchical class structure by integrating this metric in the supervision of a \textit{prototypical network}.
Our method relies on jointly learning a feature-extracting network and a set of class prototypes whose relative arrangement in the embedding space follows an hierarchical metric.
We show that this approach allows for a consistent improvement of the error rate weighted by the cost matrix when compared to traditional methods and other prototype-based strategies.
Furthermore, when the induced metric contains insight on the data structure, our method improves the overall precision as well. Experiments on four different public datasets---from agricultural time series classification to depth image semantic segmentation---validate our approach.
\end{abstract}
\section{Introduction}
Most classification models focus on maximizing the prediction accuracy, regardless of the semantic nature of errors. This can lead to high performing models, but puzzling errors such as confusing tigers and sofas, and casts doubt on what a model actually understands of the required task and data distribution. 
Neural networks in particular have been criticized for their tendency to produce improbable yet confident errors, notably when under adversarial attacks \citep{akhtar2018threat}.
\LOIC{
Training deep models to produce not only produce fewer but also \emph{better} errors can increase their trustworthiness, which is crucial for downstream applications such as autonomous driving or land use and land cover monitoring \cite{bertinetto2020making, deng2010does}.}

In many classification problems, the target classes can be organized according to a tree-shaped hierarchical structure.
Such a taxonomy can be generated by domain experts, or automatically inferred from class names using the WordNet graph \citep{miller1990introduction} or from word embeddings \citep{mikolov2013word2vec}.
A step towards more reliable and interpretable algorithms would be to explicitly model the difference of gravity between errors, as defined by a hierarchical nomenclature.

For a classification task over a set $\cK$ of $K$ classes, the hierarchy of errors can be encapsulated by a cost matrix $D\in \bR^{K\times K}_+$, defined such that the cost of predicting class $k$ when the true class is $l$ is $D[k,l]\geq 0$, and $D[k,k]=0$ for all $k=1\cdots K$. Among many other options \citep{kosmopoulos2015evaluation}, one can define $D[k,l]$ as the length of the shortest path between the nodes corresponding to classes $k$ and $l$ in the tree-shaped class taxonomy.

\begin{figure}[t!]
\footnotesize
\begin{tabular}{cccc}
\includegraphics[width=0.25\textwidth]{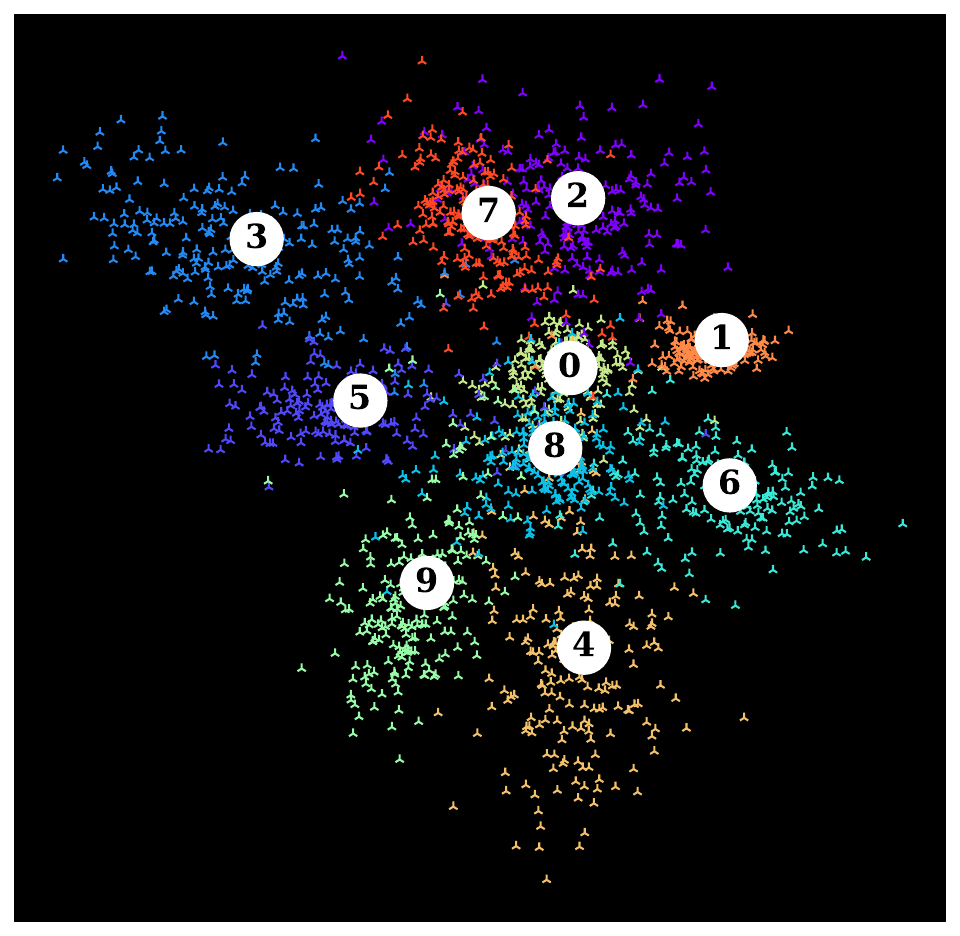}&
\includegraphics[width=0.25\textwidth]{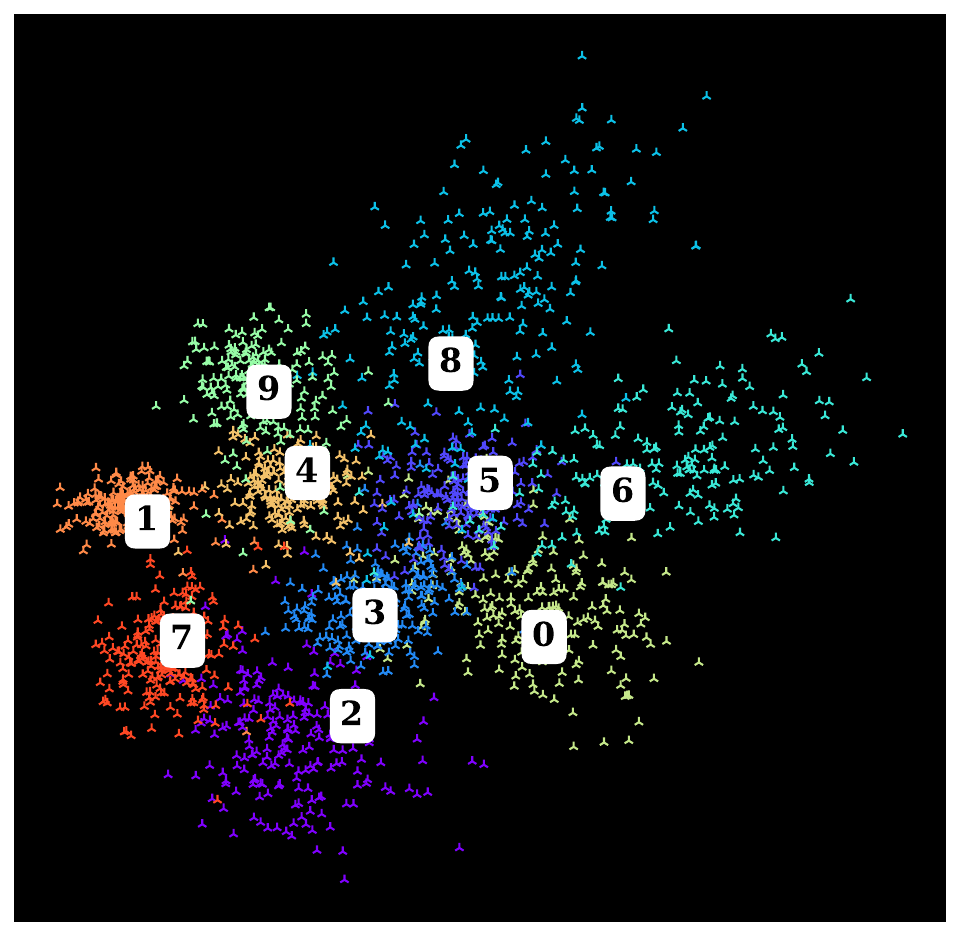}&
\includegraphics[width=0.25\textwidth]{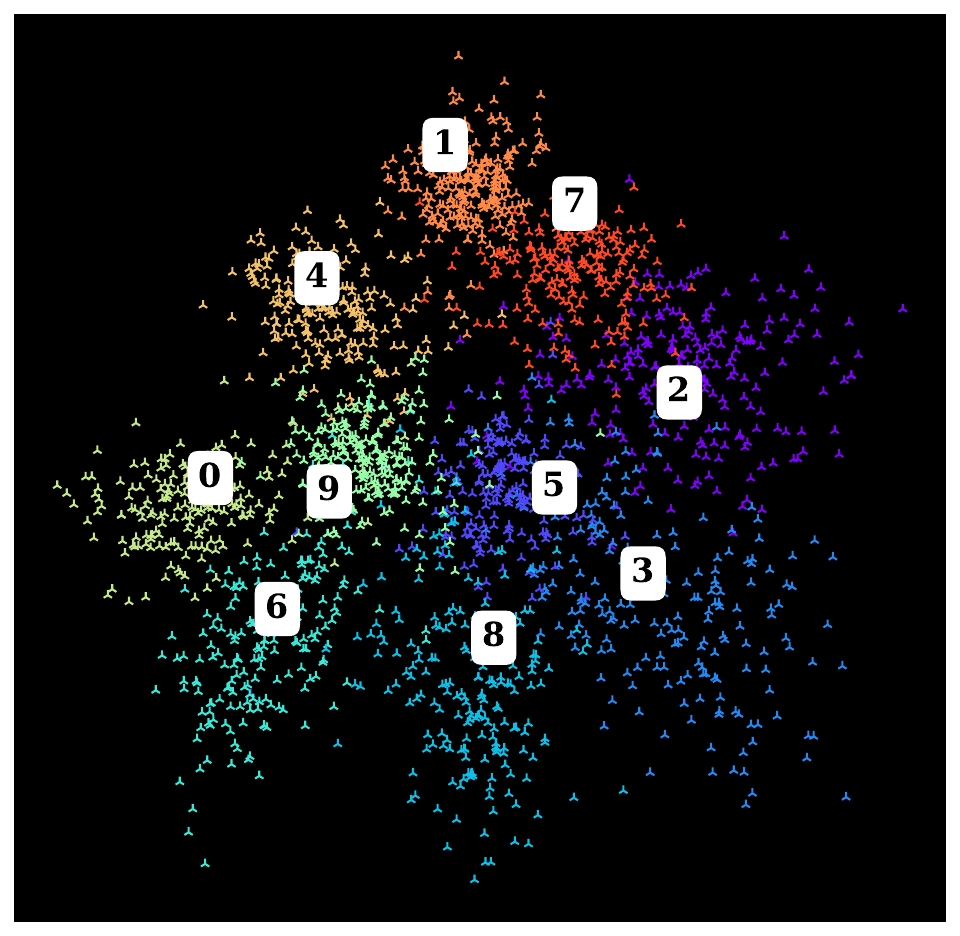}&
\hspace{-.5cm}
\includegraphics[width=0.12\linewidth, trim= 5cm 2.5cm 17.5cm 2.5cm, clip]{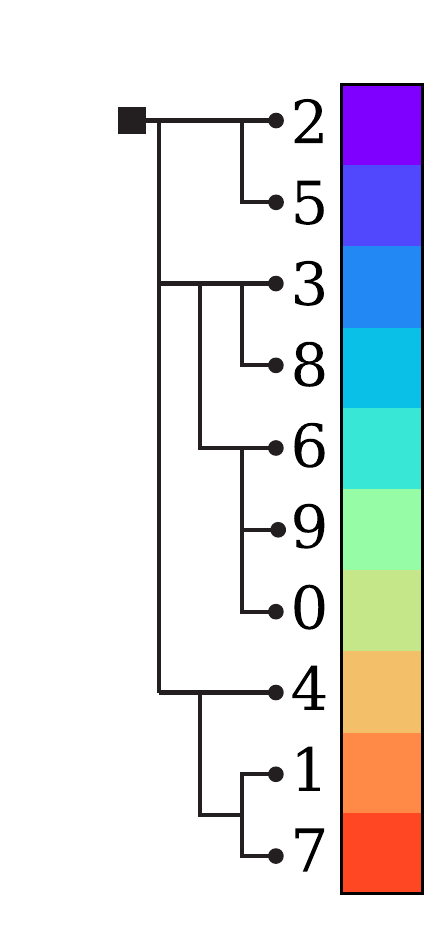}\\
(a) Cross entropy, & (b) Learnt prototypes, & (c) Guided prototypes,& \\
distortion$^1=0.47$,& distortion$=0.42$,& distortion$=0.22$,& \\
ER$=15.2\%$, AHC $=0.81$&ER$=14.2\%$, AHC $=0.75$&ER$=11.9\%$, AHC $=0.52$&\\
\end{tabular}
\vspace{.2cm}
\caption{\small{Mean class representation
	 \protect\tikz \protect\node[circle, thick, draw = black, fill = none, scale = 0.7] {};
	 , prototypes
	 \protect\tikz \protect\node[thick, draw = black, fill = none, scale = 0.9, rounded corners=0.03cm] {};
	 , and $2$-dimensional embeddings
	 \protect\begin{tikzpicture} 
	 \protect\draw [very thick, orange] (0,0.05) -- (0,0.15);
	 \protect\draw [very thick, orange] (-0.075,0) -- (0,0.05);
	 \protect\draw [very thick, orange] (+0.075,0) --(0,0.05);
	 \protect\end{tikzpicture}
	  learnt on perturbed MNIST by a 3-layer convolutional net 
	  with three different classification modules: (a) cross-entropy,
	  (b) learnt prototypes, and (c) learnt prototypes guided by a tree-shaped taxonomy (constructed according to the authors' perceived visual similarity between digits).
	The guided prototypes (d)
	embed more faithfully the class hierarchy: classes with low error cost are closer. This is associated with a decrease in the \emph{Average Hierarchical Cost} (AHC), as well as \emph{Error Rate} (ER), indicating that our taxonomy may contain useful information for learning better visual features.
	}}
\label{fig:viz}
\end{figure}

As pointed out by \citet{bertinetto2020making}, the first step towards algorithms aware of hierarchical structures would be to generalize the use of cost-based metrics.
For example, early iterations of the ImageNet challenge \citep{russakovsky2015imagenet, deng2010does} proposed to weight errors according to hierarchy-based costs. For a dataset indexed by $\cN$, the \emph{Average Hierarchical Cost} (AHC) between class predictions $y\in \cK^\cN$ and the true labels $z\in \cK^\cN$ is defined as:
\begin{equation}
\text{AHC}(y ,z) = \frac{1}{|\mathcal{N}|} \sum_{n \in \cN} D[y_n, z_n]~.
\end{equation}
Along with the evaluation metrics, the loss functions should also take the cost matrix into account.
While it is common to focus on retrieving certain classes through weighting \cite{lin2017focal, bulo2017loss} or sampling \cite{sung1996learning,shrivastava2016training} schemes, preventing confusion between specific classes is less straightforward.
For example, the cross entropy with one-hot target vectors singles out the predicted confidence for the true class, but treats all other classes equally.
Beyond reducing the AHC, another advantage of incorporating the class hierarchy into the learning phase is that $D$ may contain information about the structure of the data as well. Although it is not always the case, co-hyponyms (\ie siblings) in a class hierarchy tend to share some structural properties. Encouraging such classes to have similar representations could lead to more efficient learning, \eg by leveraging common feature detectors.
Such priors on the class structure may be especially crucial when dealing with a large taxonomy, as noted by \citet{deng2010does}. 

\setcounter{footnote}{1}
\footnotetext{For a formal definition of scale-free distortion, see \secref{sec:disto}; the distortion is computed with respect to the means of class embeddings for the cross entropy.}

In this paper, we introduce a method to integrate a pre-defined
class hierarchy into a classification algorithm.
We propose a new distortion-based regularizer for prototypical network \citep{yang2018robust, chen2019thislooks}. This penalty allows the network to learn prototypes organized so that their 
pairwise distances reflect the error cost defined by a class hierarchy.
Our contributions are as follows:
\begin{itemize}
    \item We introduce a scale-independent formulation of the distortion between two metric spaces and an associated smooth regularizer.
    \item This formulation allows us to incorporate knowledge of the class hierarchy into a neural network at no extra cost in trainable parameters and computation.
    \item We show on four public datasets (CIFAR100 , NYUDv2, S2-Agri, and iNaturalist-19) that our approach decreases the average cost of the prediction of standard backbones.
    \item As illustrated in \figref{fig:viz}, we show that our approach can also lead to a better (unweighted) precision, which we attribute to the useful priors contained in the hierarchy.
\end{itemize}
\section{Related Work}
%
\paragraph{Prototypical Networks:} Our approach builds on the growing corpus of work on prototypical networks. These models are deep learning analogues of nearest centroid classifiers \citep{tibshirani2002diagnosis} and Learning Vector Quantization networks \citep{sato1995generalized, kohonen1995learning}, which associate to each class a representation, or prototype, and classify the observations according to the nearest prototype.
These networks have been successfully used for few-shot learning \citep{snell2017prototypical, dong2018few}, zero-shot learning \citep{prototypicalpriors}, and supervized classification \citep{guerriero2018deepNCM, yang2018robust, mettes2019hyperspherical,  chen2019thislooks}.

In most approaches, the prototypes are directly defined as the centroid of the learnt representations of samples of their classes, and updated at each episode \citep{snell2017prototypical} or iteration \citep{guerriero2018deepNCM}. In the work of \citet{mettes2019hyperspherical} and \citet{prototypicalpriors}, the prototypes are defined prior to learning the embedding function. In this work, we follow the approach of \cite{yang2018robust} and {learn} the prototypes simultaneously with the data embedding function. 
\vspace{-.2cm}
\paragraph{Hierarchical Priors:}
The idea of exploiting the latent taxonomic structure of semantic classes to improve the accuracy of a model has been extensively explored \citep{silla2011survey}, from traditional Bayesian modeling \citep[Chapter~5]{gelman2013bayesian} to adaptive deep learning architectures \citep{yan2015hd, roy2020tree, salakhutdinov2012learning, ayub2011ecg}. However, for these neural networks, the hierarchy is discovered by the network itself to improve the overall accuracy of the model. In our setting, the hierarchy is defined a priori and serves both to evaluate the quality of the model and to guide the learning process towards a reduced prediction cost.

\citet{srivastava2013treepriors} propose to implement Gaussian priors on the weight of neurons according to a fixed hierarchy.
\citet{redmon2017yolo9000} implements an inference scheme based on a tree-shaped graphical model derived from a class taxonomy. 
Closest to our work, \citet{hou2016squared} propose a regularization based on the earth mover distance to penalize errors with high cost.

More recently, \citet{bertinetto2020making} highlighted the relative lack of well-suited methods for dealing with hierarchical nomenclatures in the deep learning literature. They advocate for a more widespread use of the AHC for evaluating models, and detail two simple baseline classification modules able to decrease the AHC of  deep models: \emph{Soft-Labels} and \emph{Hierarchical Cross-Entropy}. See the appendix for more details on these schemes. 
Following this objective, \citet{karthik2021no} propose a an inference-time risk minimization scheme to reduce the AHC of the predictions based on the predicted posteriors.
\vspace{-.2cm}
\paragraph{Hyperbolic Prototypes:}
Motivated by their capacity to embed hierarchical data structures into low-dimensional spaces,
\citep{de2018representation}, hyperbolic spaces are at the center of recent advances in modeling hierarchical relations \citep{nickel2017poincare, khrulkov2020hyperbolic}.
Closer to this work, \citep{liu2020hyperbolicvisual, long2020actionhyperbole} also propose to embed a class hierarchy into the latent representation space. However, both approaches embed the class hierarchy before training the data embedding network. In contrast, we argue that incorporating the hierarchical structure during the training of the model allows the network and class embeddings to share their respective insights, leading to a better trade-off between AHC and accuracy. 
In this paper, we only explore Euclidean geometry, as this setting allows for the seamless integration of our method without changing the number of bits of precision or the optimizer \cite{de2018representation}.
\vspace{-.2cm}
\paragraph{Finite Metric Embeddings:} Our objective of computing class representations with pairwise distances determined by a cost matrix has links with finding an isometric embedding of the cost matrix---seen as a finite metric. This problem has been extensively studied \citep{indyk20178, bourgain1985lipschitz} and is at the center of the growing interest for hyperbolic geometry \citep{de2018representation}. Here, our goal is simply to influence the learning of prototypes with a metric rather than necessarily seeking the best possible isometry.
%
\section{Method}
We consider a generic dataset $\cN$ of $N$ elements $x \in \cX^\cN$ with ground truth classes $z \in \cK^\cN$.
The classes $\cK$ are organized along a tree-shape hierarchical structure, allowing us to define a cost matrix $D$ by considering the shortest path between nodes.
The matrix thus defined is symmetric, with a zero diagonal, strictly positive elsewhere, and respects the triangle inequality: $D[k,l] + D[l,m] \geq D[k,m]$ for all $k,l,m$ in $\cK$. In other words, $D$ defines a finite metric. We denote by $\Omega$ an \emph{embedding space} which, when equipped with the distance function $d:\Omega\times\Omega \mapsto \mathbb{R}_+$, forms a continuous metric space.
%
\subsection{Prototypical Networks}
A prototypical network is characterized by an embedding function $f:\cX\mapsto\Omega$, typically a neural network, and a set $\pi\in\Omega^\cK$ of $K$ prototypes.
$\pi$ must be chosen such that any sample $x_n$ of true class $k$ has a representation $f(x_n)$ which is \emph{close} to $\pi_k$ and \emph{far} from other prototypes.

Following the methodology of \citet{snell2017prototypical}, a prototypical network $(f,\pi)$ associates to an observation $x_n$ the posterior probability over its class $z_n$ defined as follows:
\begin{align}
p(z_n=k| x_n)   &= 
\frac{\exp{\left(-d\left(f(x_n),\pi_{k}\right)\right)}}
{\sum_{l \in \mathcal{K}}\exp{\left(-d\left(f(x_n), \pi_{l}\right)\right)}}~, \forall k \in \cK \label{eq:proba}
\end{align}
We define an associated loss as the normalized negative log-likelihood of the true classes: 
\begin{align}
\Ldata(f,\pi)
&=
\frac1{N} \sum_{n \in \cN}
\left( d(f(x_n), \pi_{z_n})
+
\log\left(\sum_{l\in \cK} \exp{(-d(f(x_n), \pi_{l}))}\right)\right)~.
\label{eq:ldata_long}
\end{align}

This loss encourages the representation $f(x_n)$ to be close to the prototype of the class $z_n$ 
and far from the other prototypes.
Conversely, the prototype $\pi_k$ is drawn towards the representations $f(x_n)$ of samples $n$ with true class $k$, and away from the  representations of samples of other classes.

Following the insights of \cite{yang2018robust}, the embedding function $f$ and the prototypes $\pi$ are learned {simultaneously}. This differs from many works on prototypical networks which learn prototypes separately or define them as centroids of representations.
We take advantage of this joint training to learn prototypes which take into account both the distribution of the data and the relationships between classes, as described in the next section.
\subsection{Metric-Guided Penalization}
\label{sec:disto}
We propose to incorporate the cost matrix $D$ into a regularization term in order to encourage the prototypes' positions in the embedding space $\Omega$ to be consistent with the finite metric defined by $D$.
Since the sample representations are attracted to their respective prototypes in \eqref{eq:ldata_long}, such regularization will also affect the embedding network.
\vspace{-.2cm}
\paragraph{Metric Distortion}
As described in \citet{de2018representation}, the distortion of a mapping $k \mapsto \pi_k$ between the finite metric space $(\cK, D)$ and the continuous metric space $(\Omega,d)$ can be defined as the average relative difference between distances in the source and target space:
\begin{align}
\text{disto}(\pi,D) &= \frac1{K(K-1)}
\sum_{k,l \in \cK^2,\,k\neq l} \frac{\left\vert d(\pi_{k},\pi_{l}) - D[k,l]\right\vert}{D[k,l]}~.\label{eq:disto}
\end{align}
We argue that a network $(f,\pi)$ trained to minimize $\Ldata$ and whose prototypes $\pi$ have a low distortion with respect to $D$ should produce errors with low 
hierarchical costs. 
To understand the intuition behind this idea, let us consider a sample $x_n$ of true class $k$ and  misclassified as class $l$.  
This tells us that the distance between $f(x_n)$ and $\pi_l$ is small.
If $k$ and $l$ have a high cost according to $D$, and since $k\mapsto \pi_k$ is of low distortion, then $d(\pi_k,\pi_l)$ must be large. The triangular inequality tells us that 
$d(f(x_n),\pi_k) \geq d(\pi_k,\pi_l) - d(f(x_n),\pi_l)$, and consequently that $d(f(x_n), \pi_k)$ must be large as well, which contradicts that $(f,\pi)$ minimizes $\Ldata$.


\vspace{-.2cm}
\paragraph{Scale-Free Distortion}
For a prototype arrangement $\pi$ to have a small distortion with respect to a finite metric $D$ as defined in \equaref{eq:disto}, the distance between prototypes must correspond to the distance between classes. This imposes a specific scale on the distances between prototypes in the embedding space.
This scale may conflict with the second term of $\Ldata$ which encourages the distance between embeddings and unrelated prototypes to be as large as possible. Therefore, lower distortion may also cause lower precision. To remove this conflicting incentive, we introduce a scale-independent formulation of the distortion \eqref{eq:scalefree} where $s \cdot \pi$ are the scaled prototypes, whose coordinates in $\Omega$ are multiplied by a scalar factor $s$. As shown in the appendix, $\text{disto}^{\text{scale-free}}$ can be efficiently computed algorithmically.

\begin{align}
 \text{disto}^{\text{scale-free}}(\pi,D) &= \min_{s \in \bR_+}{\text{disto}(s \cdot \pi, D)}~,
 \label{eq:scalefree}
\end{align}
\vspace{-.2cm}
\paragraph{Distortion-Based Penalization}
We propose to incorporate the error qualification $D$ into the prototypes' relative arrangement by encouraging a low \emph{scale-free} distortion between $\pi$ and $D$. To this end, we define $\Ldisto$, a smooth surrogate of $\text{disto}^{\text{scale-free}}$ \eqref{eq:loss:distosmooth}. as detailed in the appendix,  $\Ldisto$ can be computed in closed form as a function of $\pi$ and can be directly used as a regularizer. 
\begin{align}
\label{eq:loss:distosmooth}
    \Ldisto(\pi) &= \frac{1}{K(K-1)} \min_{s \in \bR_+} 
\sum_{k,l \in \cK^2,\,k\neq l} \left(\frac{ s 
d(\pi_{k},\pi_{l}) - D[k,l]}{D[k,l]}\right)^2 
~.
\end{align}
 
\subsection{End-to-end Training}
We combine $\Ldata$ and $\Ldisto$ in a single loss $\La$. $\Ldata$  allows to jointly learn the embedding function $f$ and the class prototypes $\pi$, while $\Ldisto$  enforces a metric-consistent prototype arrangement, with $\lambda\in \bR_+$ an hyper-parameter setting the strength of the regularization:
\begin{equation}
    \La(f, \pi) = \Ldata(f, \pi) + \lambda 
    \Ldisto(\pi)~. \label{eq:loss:total}
\end{equation}

%
\section{Experiments}
\subsection{Datasets and Backbones}
%
\begin{table}[h]
\caption{\small{Data composition and taxonomies of the four studied datasets. IR stands for the Imbalance Ratio (largest over smallest class count), nodes and leaves denote respectively the total number of classes and leaf-classes in the tree-shape hierarchy, ABF stands for the Average Branching Factor, and $\langle D\rangle$ stands for the average pairwise distance.}}
\centering
\vspace{.2cm}
\small{
\begin{tabular}{@{}l r  r r c c  c c  r @{}}
\multirow{2}{*}{Dataset} & \multicolumn{3}{c}{Data} & \phantom{abc}&\multicolumn{3}{c}{Hierarchical Tree} \\ \cmidrule{2-4} \cmidrule{6-8}
 & Volume (Gb) &{Samples}& IR &  &Depth & Nodes (leaves)& ABF & $\langle D\rangle$\\\midrule
NYUDv2  & 2.8 & 1449&  93&& 3 & 57 (40) & 5.0 & 4.3 \\
S2-Agri  & 28.2 & 189$\;$971& 617 && 4 & 83 (45) & 5.8 & 6.5\\
CIFAR100  & 0.2 & 60$\;$000&1 & &5 & 134 (100) &3.8 & 7.0\\
iNat-19  & 82.0 & 265$\;$213&31 && 7 & 1189 (1010) & 6.6 & 11.0\\
\bottomrule
\end{tabular}
}
\label{tab:nomenc}
\end{table}
\label{sec:datasets}
We evaluate our approach with different tasks and public datasets with fine-grained class hierarchies: image classification on CIFAR100 \citep{cifar100} and iNaturalist-19 \citep{van2018inaturalist}, RGB-D image segmentation  on NYUDv2 \citep{nyuv2}, and image sequence classification on S2-Agri \citep{sainte2019satellite}.
We define the cost matrix of these class sets as the length of the shortest path between nodes in the associated tree-shape taxonomies represented in the Appendix.
As shown in \tabref{tab:nomenc}, these datasets cover different settings in terms of data distribution and hierarchical structure.
\vspace{-.5cm}
\paragraph{Illustrative Example on MNIST:} \VIVIEN{In \figref{fig:viz}, we illustrate the difference in performance and embedding organization of the embedding space for different approaches. We use a small 3-layer convolutional net trained on MNIST with random rotations (up to $40$ degrees) and affine transformations (up to $1.3$ scaling). For plotting convenience, we set the features' dimension to $2$.}
\vspace{-.3cm}
\paragraph{Image Classification on CIFAR100:}
We use a super-class system inspired by \citet{cifar100} and form a $5$-level hierarchical nomenclature of size:  $2$, $4$, $8$, $20$, and $100$ classes. We use as backbone the established ResNet-18 \citep{he2016deep} as embedding network for this dataset.
\vspace{-.3cm}
\paragraph{RGB-D Semantic Segmentation on NYUDv2:}
We use the standard split of $795$ training and $654$ testing pairs.
We combine the $4$ and $40$ class nomenclatures of \citet{gupta2013perceptual} and the $13$ class system defined by \citet{handa2015scenenet}
to construct a $3$-level hierarchy. We use FuseNet \citep{hazirbas2016fusenet} as backbone for this dataset. 
\vspace{-.3cm}
\paragraph{Image Sequence Classification on S2-Agri:}
S2-Agri is composed of $189\;971$ sequences of multi-spectral satellite images of agricultural parcels. 
We define a $4$-level crop type hierarchy of size $4$, $12$, $19$, and $44$ classes with the help of experts from a European agricultural monitoring agency (ASP).
We use the PSE+TAE architecture \citep{sainte2019satellite} as backbone, and follow their $5$-fold cross-validation scheme for training. \LOIC{Crop mapping in particular benefits from predictions with a low hierarchical cost. Indeed, payment agencies  monitor the allocation of agricultural subsidies  and whether crop rotations follow best practice recommendations \cite{grant1997common}. The monetary and environmental impact of misclassifications are typically reflected in the class hierarchy designed by domain experts \cite{brankatschk2015modeling, bullock1992crop}. By achieving a low AHC, we ensure that these downstream tasks can be meaningfully realized from the predictions.}
\vspace{-.3cm}
\paragraph{Fine-Grained Image Classification on iNaturalist-19 (iNat-19)} 
iNat-19 \citep{van2018inaturalist} contains $1\,010$ different classes organized into a $7$ level hierarchy with respective width $3$, $4$, $9$, $34$, $57$, $72$, and $1\,010$. We use ResNet-18 pretrained on ImageNet as backbone. We sample $75\%$ of available images for training, while the rest is evenly split into a validation and test set.
\subsection{Hyper-Parameterization}
The embedding space $\Omega$ is chosen as $\bR^{512}$ for iNat-19 and $\bR^{64}$ for all other datasets. 
We chose $d$ as the Euclidean norm.
(see \ref{sec:ablation} for a discussion on this choice).
We evaluate our approach (Guided-proto) with $\lambda=1$ in \eqref{eq:loss:total} for all datasets.
We use the same training schedules and learning rates as the backbone networks in their respective papers.
In particular, the class imbalance of S2-Agri is handled with a focal loss \citep{lin2017focal}.
\subsection{Competing methods}
\label{sec:compmet}
In the paper where they are introduced, all backbone networks presented in \secref{sec:datasets} use a linear mapping between the sample representation and the class scores, as well as the cross-entropy loss.
The resulting performance defines a baseline, denoted as \texttt{Cross-Entropy}, and is used to estimate the gains in Average Hierarchical Cost (AHC) and Error Rate (ER) provided by different approaches.
{We reimplemented other competing methods: Hierarchical Cross-Entropy (\texttt{HXE}) \cite{bertinetto2020making}, Soft Labels \cite{bertinetto2020making}, Earth Mover Distance regularization (\texttt{XE+EMD}) \cite{hou2016squared}, Hierarchical Inference (\texttt{YOLO}) \cite{redmon2017yolo9000}, Hyperspherical Prototypes (\texttt{Hyperspherical}-\texttt{proto}) \cite{mettes2019hyperspherical}, and Deep Mean Classifiers (\texttt{Deep-NCM}) \cite{guerriero2018deepNCM}. See the Appendix for more details on these methods. Lastly, we evaluate simple prototype learning (\texttt{Learnt-proto}) \cite{yang2018robust} by setting $\lambda=0$ in \eqref{eq:loss:total}}.
\subsection{Analysis}
\begin{figure}[t]
    \centering
    \begin{tikzpicture}
    \node[anchor=south west,inner sep=0] at (0,0) {\includegraphics[width=\linewidth, trim=3.5cm 1cm 3.5cm 1cm , clip]{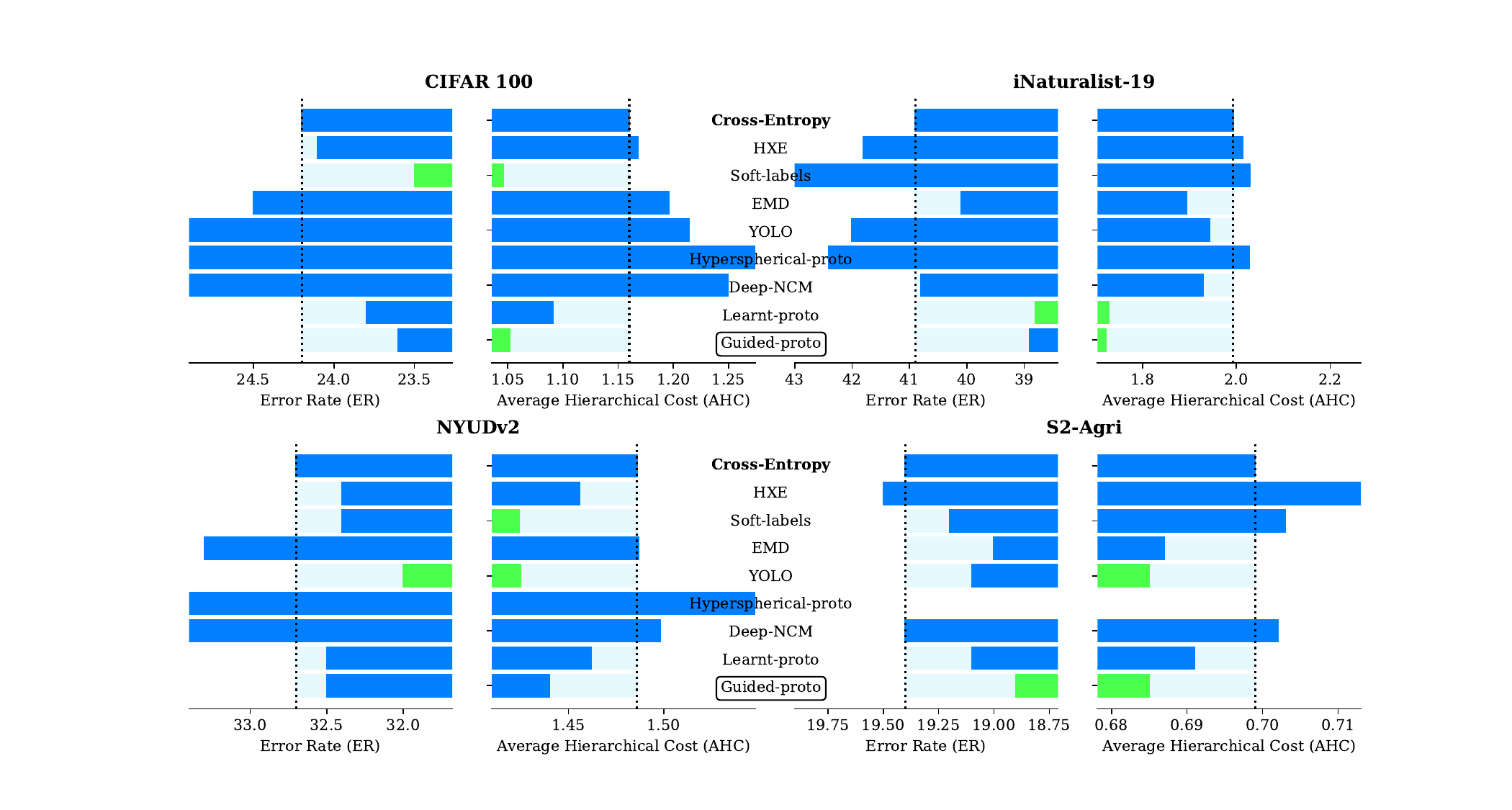}};
    \node[draw=none, fill=none, text=black,scale=0.7] at (9.85,
    1.78) {$\star$};
    \node[draw=none, text=black, scale=0.65] at (7.4, .88)  {\RS{Ilir}};
    \node[draw=none, text=black, scale=0.65] at (7.65, 1.78) {\RS{Ilir}};
    \node[draw=none, text=black, scale=0.65] at (7.0, 2.08) {\RS{Ilir}};
    \node[draw=none, text=black, scale=0.65] at (7.0, 2.38) {\RS{Ilir}};
    \node[draw=none, text=black, scale=0.65] at (7.2, 2.68) {\RS{Ilir}};
    \node[draw=none, text=black, scale=0.65] at (7.0, 2.98) {\RS{Ilir}};
    
    \node[draw=none, text=black, scale=0.65] at (7.4, 4.52) {\RS{Ilir}};
    \node[draw=none, text=black, scale=0.65] at (7.65, 5.42)  {\RS{Ilir}};
    \node[draw=none, text=black, scale=0.65] at (7.0, 5.72) {\RS{Ilir}};
    \node[draw=none, text=black, scale=0.65] at (7.0, 6.02) {\RS{Ilir}};
    \node[draw=none, text=black, scale=0.65] at (7.2, 6.32) {\RS{Ilir}};
    \node[draw=none, text=black, scale=0.65] at (7.0, 6.62) {\RS{Ilir}};
\end{tikzpicture}
\vspace{.1cm}
\caption{\small{Error Rate (ER) in \% and Average Hierarchical Cost (AHC) on four datasets for \texttt{Guided-proto}, the \texttt{Cross-Entropy} baseline (in bold), and competing approaches. \VIVIEN{Methods that use the hierarchical knowledge are indicated with the symbol \protect\tikz \protect \node[draw=none, text=black, scale=0.5] at (0, 0) {\protect\RS{Ilir}};} . The best performances on each dataset are plotted in green. Our guided prototype approach improves both the ER and AHC across the four datasets compared to the baseline. The metrics are computed with the median over $5$ runs for CIFAR100, the average over $5$ cross-validation folds for S2-Agri, and a single run for NYUDv2 and iNat-19. The numeric values are given in the Appendix. ($\star$: not evaluated).}} 
    \label{fig:results}
\end{figure}
\paragraph{Overall Performance:} As displayed in \figref{fig:results}, the benefits provided by our approach can be appreciated on all datasets. Compared to the \texttt{Cross-Entropy} baseline, our model improves the AHC by $3\%$ on NYUDv2 and S2-Agri, and up to $9$\% and $14$\% for CIFAR100, and iNat-19 respectively.
The hierarchical inference scheme \texttt{YOLO} of \citet{redmon2017yolo9000} performs on par or better than our methods for NYUDv2 and S2-Agri, while \texttt{Soft-labels} perform well on CIFAR100 and NYUDv2. \LOIC{Yet, metric guided prototypes brings the most consistent reduction of the hierarchical cost across all  tasks, datasets, and class hierarchies  configurations.}
This suggest that arranging the embedding space consistently with the cost metric is a robust way of reducing a model's hierarchical error cost. We argue that these results, combined with its ease of implementation, make a strong case for our approach.

While being initially designed to 
reduce the AHC, our method also 
provides a relative decrease of the ER
by $3$ to $4$\% across all datasets compared to the cross-entropy baseline. This indicates that cost matrices derived from the class hierarchies can indeed help neural networks to learn richer representations.
\vspace{-.5cm}
\paragraph{Prototype Learning:}
We observe that the learnt prototype approach \texttt{Learnt-proto} consistently outperforms the \texttt{Deep-NCM} method.
This suggests that defining prototypes as the centroids of their class representations might actually be disadvantageous. As illustrated on \figref{fig:viz}, the positions of the embeddings tend to follow a Voronoi partition \citep{fortune1992voronoi} with respect to the learnt prototypes of their true class rather than prototypes being the centroid of their associated representations.
A surprising observation for us is that \texttt{Learnt-proto} consistently outperforms the \texttt{Cross-Entropy} baseline, both in terms of AHC and ER.
%
\paragraph{Computational Efficiency:} Computing distances between representations and prototypes is comparable in terms of complexity  than computing a linear mapping. \LOIC{The scaling factor in $\Ldisto$ can be efficiently obtained as described in the Appendix.}
\LOIC{In practice, we observed that both training and inference time are identical for \texttt{Cross-Entropy} and \texttt{Guided-proto}: most of the time is taken by the computation of the embeddings.}

\begin{figure}
    \centering
    \subfigure[Cross-Entropy]{
    \centering
    \includegraphics[width=.30\linewidth, trim=1cm 1.8cm .9cm 0cm, clip]{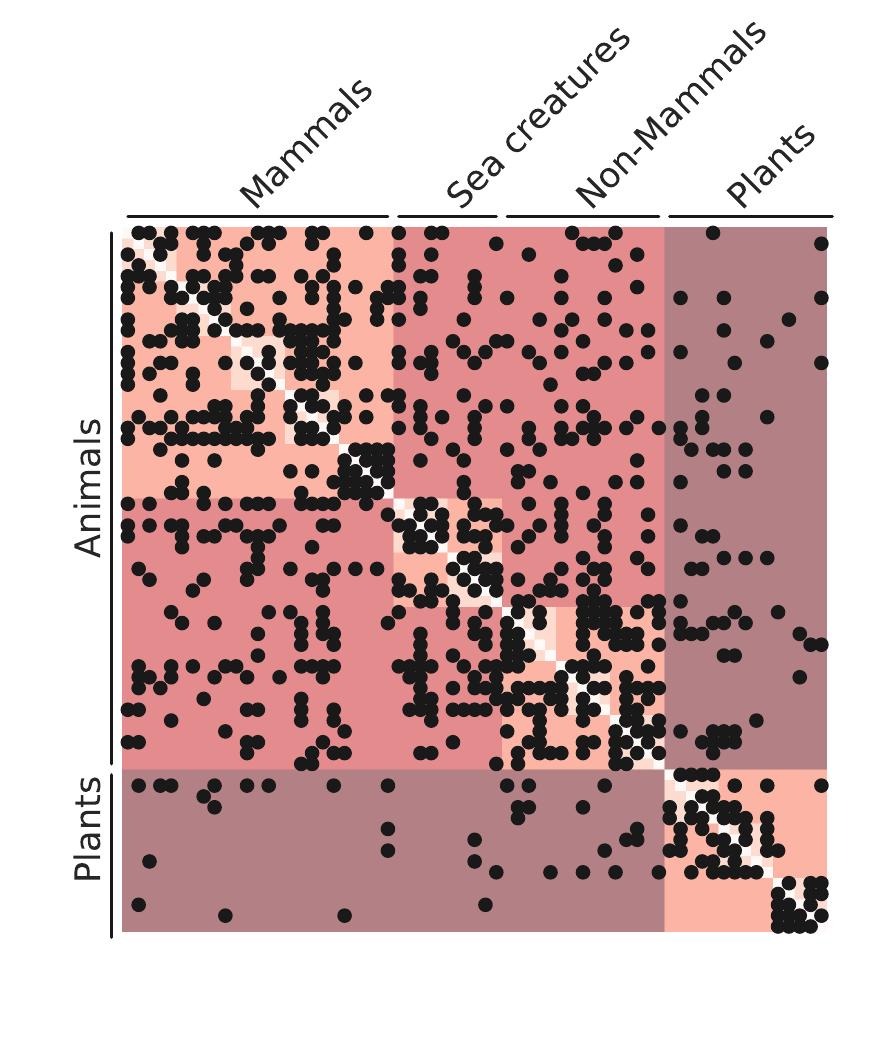}
    } 
    \hfill
    \subfigure[Metric-Guided Prototypes]{
    \centering
    \includegraphics[width=.30\linewidth, trim=1cm 1.8cm .9cm 0cm, clip]{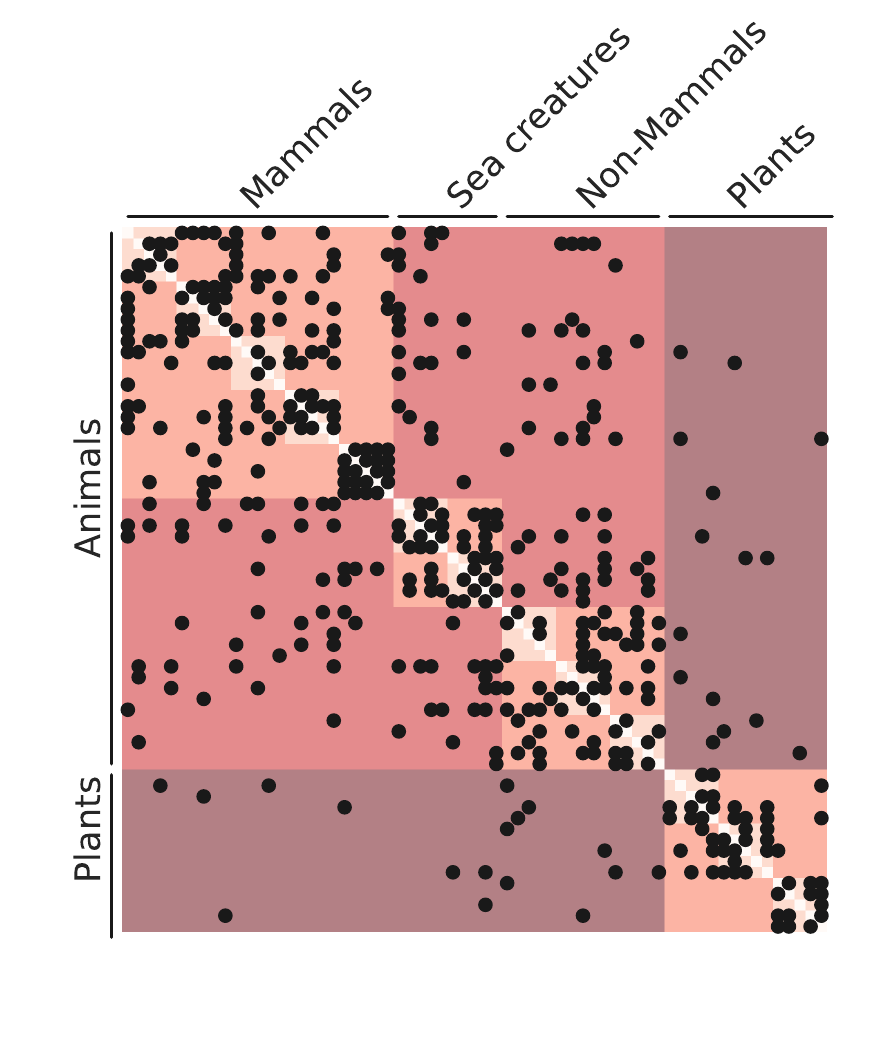}
    } 
    \hfill
    \subfigure[Hierarchical cost]{
    \qquad
    \includegraphics[width=.08\linewidth, trim=0cm 0cm 0cm 0cm, clip]{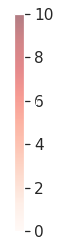}
    \qquad
    } 
    \vspace{.2cm}
    \caption{\LOIC{Partial confusion matrix for the ``living organism'' class subset of CIFAR100 for the Cross-Entropy baseline (a) and our approach (b). For readability, we only display (in black) entries of the matrices with at least one confusion. We also represent the cost of confusing different classes in shades of reds (c). We note that our approach yields fewer confusions between pairs of classes with high costs, such as plants and animals.}}
    \label{fig:confusion}
\end{figure}

\subsection{Restricted Training Data Regime}
\begin{minipage}{\linewidth}
\begin{wrapfigure}{r}{0.47\textwidth}
 \centering
    \vspace{-.6cm}
    \includegraphics[width=\linewidth]{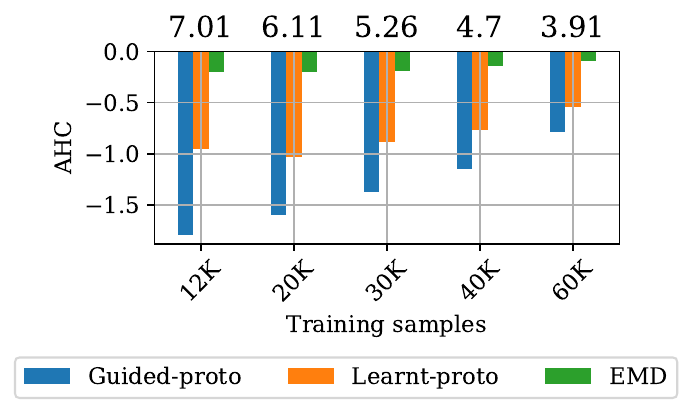}
    \vspace{.1cm}
    \caption{AHC of ResNet-18 trained on restricted training sets of iNaturalist-19 with \texttt{Guided-proto}, \texttt{Learnt-proto}, and \texttt{EMD}. 
    We represent the relative improvement compared to the performance of the \texttt{Cross-Entropy} baseline, which is shown on top of the plots.}
    \vspace{-.1cm}
    \label{fig:ldr_inat}
\end{wrapfigure}

We observed that the \texttt{Learnt-proto} method decreases the AHC across all four datasets even though it does not take the cost matrix into account. This suggests that, given enough data, this simple model can 
learn an empirical taxonomy through its prototypes' arrangement. Furthermore, this taxonomy can share enough similarity with the one designed by experts to result in a decrease in AHC.
To further evaluate the benefit of explicitly using the expert taxonomy with our approach, we train the models \texttt{Learnt-proto}, \texttt{Guided-proto}, and \texttt{EMD} with only part of the $160$k images in the training set of iNat-19, and without pretraining on ImageNet.
To compensate for the lack of data, we increase the regularization strength to $\lambda=20$.
\end{minipage}

In \figref{fig:ldr_inat}, we observe that the two prototype-based approaches consistently improve the performance of the baseline for all training set sizes in terms of AHC. Moreover, the advantages brought by our proposed regularization are all the more significant when applied to small training sets.
This observation reinforces the idea that the \texttt{learnt-proto} method requires large amounts of data to learn a meaningful class hierarchy in an unsupervized way.

\subsection{Ablation Study}
\vspace{-.4cm}
In the Appendix, we present an extensive ablation study. We present here its main take-away.
\vspace{-.9cm}
\paragraph{Scale-Free Distortion:}
Our method for automatically choosing the best scale in our smooth distortion surrogate leads to an improvement of $0.9$ ER on the iNat-19 dataset, which amounts to half the improvement compared to the baseline. In the other datasets, the improvements were more limited. 
We attribute the impact 
of our scale-free distortion 
on iNat-19 in particular to the structure of its class hierarchy: 
at the lowest level
, iNat-19 classes have on average $14$ co-hyponyms (siblings), compared to only $2$ to $5$ for the other datasets. 
When minimizing the distortion with a fixed scale of $1$, the prototypes of hyponyms are incentivized to be close with respect to $d$ since hyponyms have a small hierarchical distance of $2$.
This clashes with the minimization of the second part of $\Ldata$ as defined in $\eqref{eq:ldata_long}$, which mutually repels prototypes of different classes. This conflict, made worse by classes with many hyponyms, is removed by our scale-free distortion. \VIVIEN{See the Appendix 
for additional insights into how our automatic scaling addresses this issue. }
\vspace{-.6cm}
\paragraph{Choice of Metric Space:}
Prototypical networks operating on $\Omega=\bR^m$ typically use the squared Euclidean norm in the distance function, motivated by its quality as a Bregman divergence  \citep{snell2017prototypical}.
However, given the large distance between prototypes induced by our regularization, this metric can cause stability issues. We observe for all datasets that that defining $d$ as the Euclidean norm yields significantly better results across all datasets. 
\vspace{-.6cm}
\paragraph{Guided vs. fixed prototypes :} As suggested by the lower performance of \texttt{Hypersphe-} \texttt{rical-proto}, jointly learning the prototypes and the embedding network can be advantageous. To confirm this observation, we altered our \texttt{Guided-proto} method to first learn the prototypes and then the embedding network. We observed a significant decrease in performance across the board, up to $5$ more points of ER in iNat-19. \VIVIEN{Conversely, we altered  \texttt{Hyperspherical-proto} to learn spherical prototypes together with the embedding network. This improved the performance of \texttt{Hyperspherical-proto} even though it remained worse than the \texttt{Cross-Entropy} baseline ($+1.20$ ER, $+0.10$ AHC on CIFAR100).} These observations suggest that insights from the data distribution can benefit the positioning of prototypes, and that they should be learned conjointly.
%
\label{sec:ablation}
%

\section{Conclusion}
We introduced a new regularizer modeling the hierarchical relationships between the classes of a nomenclature.  
This approach can be incorporated into any classification network at no computational cost and with very little added code.
We showed that our method consistently decreased the average hierarchical cost of three different backbone networks on different tasks and four datasets. Furthermore, our approach can reduce the rate of errors as well. In contrast to most recent works on hierarchical classification, we showed that this joint training is beneficial compared to the staged strategy of first positioning the prototypes and then training a feature extracting network. 
A PyTorch implementation of our framework as well as an illustrative notebook are available at \url{https://github.com/VSainteuf/metric-guided-prototypes-pytorch}.


\ARXIV{\newpage}{}


\bibliography{mrp}

\newpage
\ARXIV{\section*{  \textbf{Supplementary Materials} -  Leveraging Class Hierarchies with Metric-Guided Prototype Learning }

\section{Notebook and illustration}
In \figref{fig:viz:app}, we represent the embeddings and prototypes generated by  variations of our networks as well as their respective performance. 
We note that the \emph{fixed} prototypes approach performs significantly worse than our metric-guided method. We observe that the resulting prototypes are more compact when they are learned independently, which can lead to an increase in misclassification.
We also remark that when the hierarchy contains no useful information, such as the arbitrary order of digits, the metric-based approach has a worse performance than the free (unguided) method. This is particularly drastic for the fixed prototype approach.

An illustrated notebook to reproduce this figure can be accessed at the following URL:

\url{https://colab.research.google.com/drive/1VoQfBx5q5lWFev0cwxLZ0qQOZU7Rlmb_#offline=true&sandboxMode=true}

To run this notebook locally, you can also download it from our repository:\\
\url{https://github.com/VSainteuf/metric-guided-prototypes-pytorch}.

\begin{figure}[h!]
\begin{tabular}{rl}
\begin{tabular}{cc}
	\subfigure[Cross entropy, ER$=15.2\%$
		disto$_{\text{vis}}=0.47$,
		disto$_{\text{abs}}=0.61$
		 AHC$_{\text{vis}}=0.81$,
		 AHC$_{\text{abs}}=0.65$]{
	\captionsetup{justification=centering}
		\centering
		\includegraphics[width=0.3\textwidth]{gfx/mnistv3/mnist_xe.pdf}
		\label{fig:vizapp:xe}
	}&
	\subfigure[Learnt prototypes, ER$=14.2\%$ 
		disto$_{\text{vis}}=0.42$,
		disto$_{\text{abs}}=0.58$
		AHC$_{\text{vis}}=0.75$,
		AHC$_{\text{abs}}=0.50$]{
	\captionsetup{justification=centering}
		\centering
		\includegraphics[width=0.3\textwidth]{gfx/mnistv3/mnist_free.pdf}
		\label{fig:vizapp:freeproto}
	}\\
	\subfigure[Guided prototypes, ER$=12.8\%$ 
		disto$_{\text{vis}}=0.22$
		AHC$_{\text{vis}}=0.56$
		]{
	\captionsetup{justification=centering}
		\centering
		\includegraphics[width=.3\textwidth]{gfx/mnistv3/mnist_gp_Dvis.pdf}
		\label{fig:vizapp:protshape}
	}
	&
	\subfigure[
		Fixed prototypes, ER$=21.5\%$
		disto$_{\text{vis}}=0.17$
		AHC$_{\text{vis}}=0.82$]{
	\captionsetup{justification=centering}
		\centering
		\includegraphics[width=.3\textwidth]{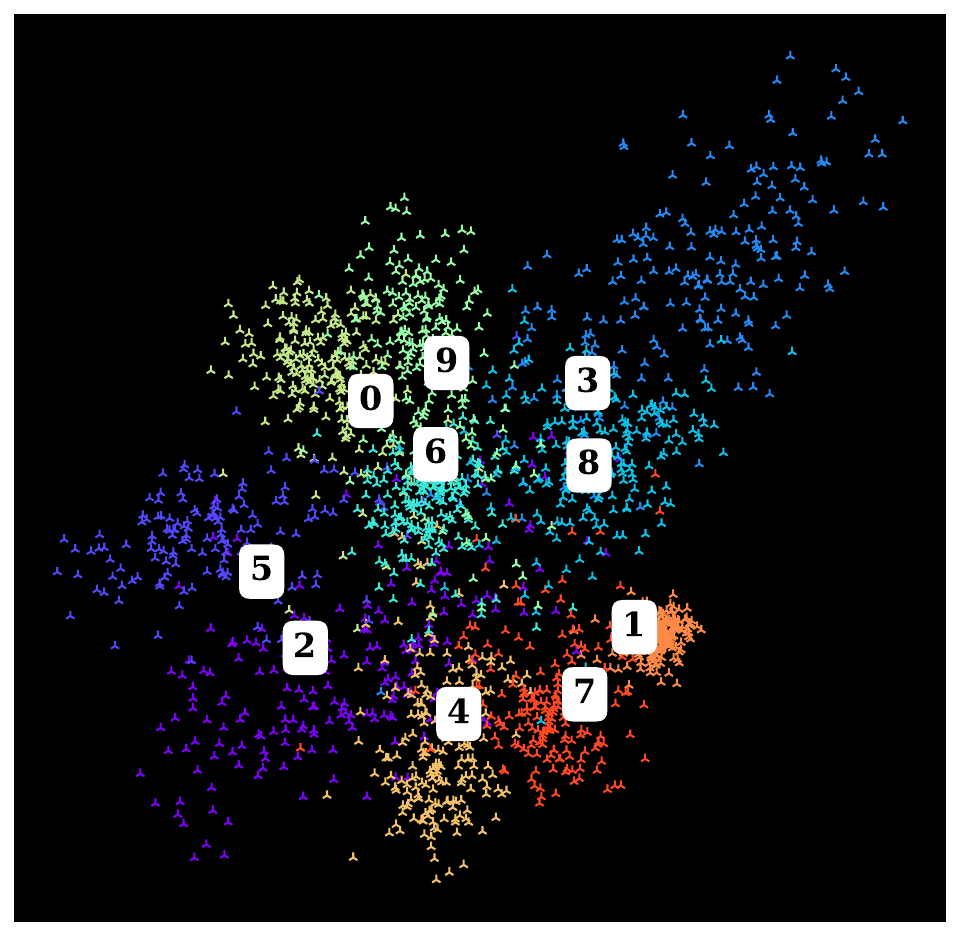}
		\label{fig:vizapp:protFshape}
	}
	\\
	\subfigure[
		Guided prototypes, ER$=16.9\%$
		disto$_{\text{abs}}=0.24$
		AHC$_{\text{abs}}=0.54$]{
	\captionsetup{justification=centering}
		\centering
		\includegraphics[width=.3\textwidth]{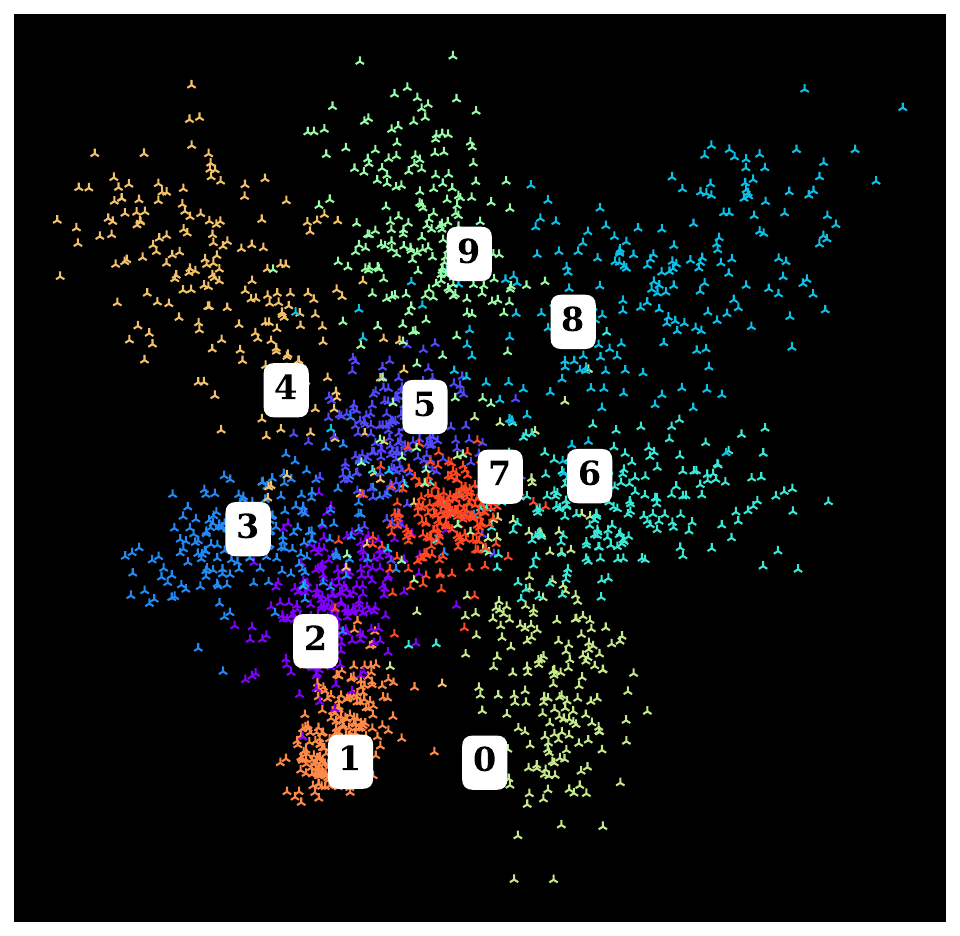}
		\label{fig:vizapp:protabs}
	}
	&
	\subfigure[Fixed prototypes, ER$=48.8\%$
		disto$_{\text{abs}}=0.00$
		AHC$_{\text{abs}}=0.80$]{
	\captionsetup{justification=centering}
		\centering
		\includegraphics[width=.3\textwidth]{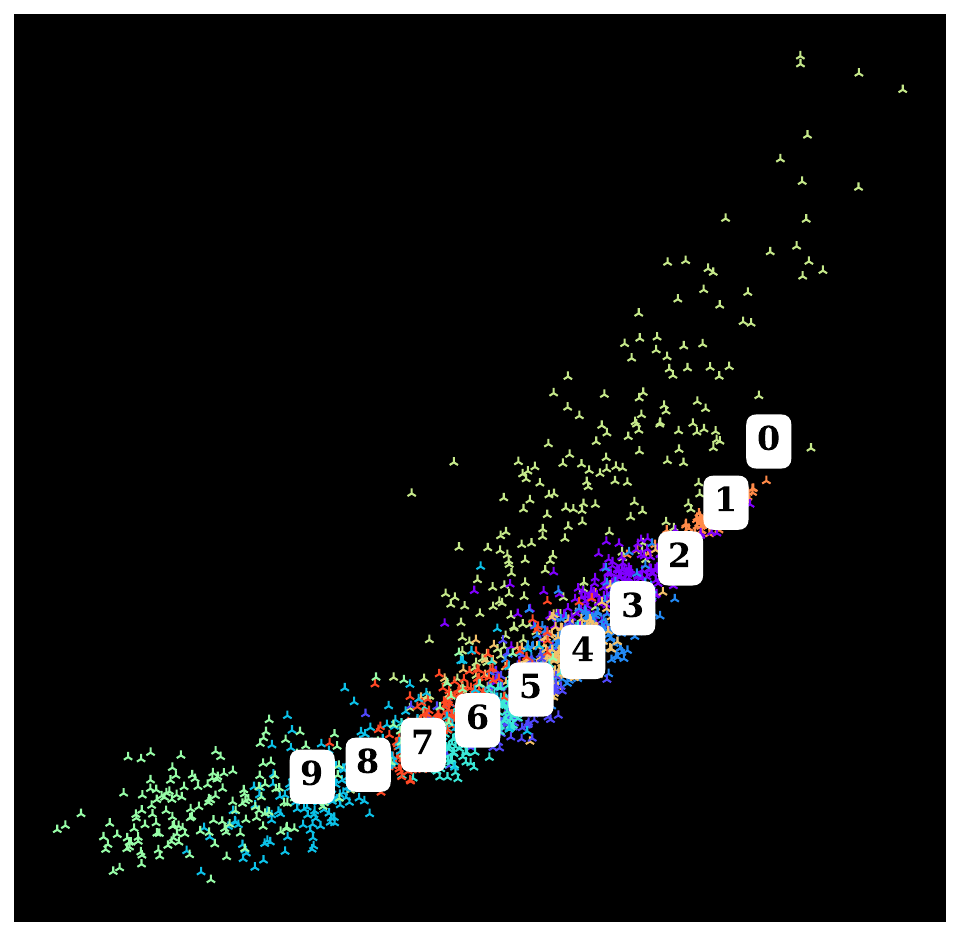}
		\label{fig:vizapp:protFabs}
	}
	\end{tabular}
	&
	\begin{tabular}{l}
	\subfigure[]{
		\centering
		\includegraphics[width=.2\textwidth, trim= 5cm 2.5cm 17.5cm 2.5cm, clip]{gfx/MNIST_tree2.pdf}
	}\\
	\subfigure[]{
		\centering~~\qquad\quad
		\includegraphics[width=.2\textwidth, trim= 0cm 5cm 22cm 0cm, clip]{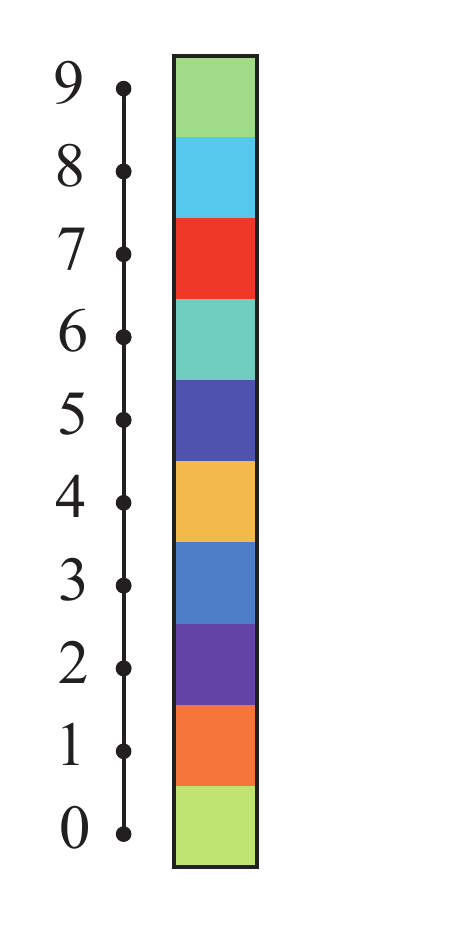}
	}
	\end{tabular}
\end{tabular}
\vspace{.1cm}
\caption[]{
	 Mean class representation
	 \protect\tikz \protect\node[circle, thick, draw = black, fill = none, scale = 0.7] {};
	 , prototypes
	 \protect\tikz \protect\node[thick, draw = black, fill = none, scale = 0.9, rounded corners=0.03cm] {};
	 , and $2$-dimensional embeddings
	 \protect\begin{tikzpicture} 
	 \protect\draw [very thick, orange] (0,0.05) -- (0,0.2);
	 \protect\draw [very thick, orange] (-0.1,0) -- (0,0.05);
	 \protect\draw [very thick, orange] (+0.1,0) --(0,0.05);
	 \protect\end{tikzpicture}
	  learnt on perturbed MNIST by a 3-layer convolutional net 
	  with six different classification modules: (a) cross-entropy, (b) learnt prototypes, (c) learnt prototypes guided by a visual taxonomy,
      (d) fixed prototypes (see \secref{sec:numres}) from a visual taxonomy ,
	  (e) learnt prototypes guided by the numbers' values, and
	  (f) fixed prototypes from the numbers' values.
	  The visual hierarchy is represented in (g) and the numerical order in (h).
       AHC$_{\text{vis}}$ corresponds to the cost defined by our proposed visual hierarchy, while AHC$_{\text{abs}}$ is defined after the chain-like  structure obtained when organizing the digits along their numerical values.
       While embedding the metric with prototypes prior to learning the representations leads to lower (scale-free) distortion, this translates into worst performance in terms of AHC and ER. Joint learning achieves better performance on both evaluation metrics.
       We also remark that when the hierarchy is arbitrary (e-f), metric guiding is detrimental to precision.
       }
	\label{fig:viz:app}
\end{figure}

\section{Additional methodological details}
\subsection{Scale-Independent Distortion}
 Computing the scale-free distortion defined in \equaref{eq:scalefree} amounts to finding a minimizer of the following function $f:\bR\mapsto\bR$:
 \begin{align}
     f(s) 
     & = \sum_{i \in I} \abs{s \alpha_{i} -1},
     \label{eq:minimzation}
 \end{align}
 with $\alpha_{k,l}=d(\pi_k, \pi_l)/D[k,l]\geq  0$, and $I$ an ordering of $\{k,l\}_{k,l \in \cK^2}$ such that the sequence $[\alpha_{i}]_{i \in I}$ is 
 non-decreasing.
 \begin{prop} A global minimizer of $f$ defined in \eqref{eq:minimzation} is given by $s^\star=1/\alpha_{k^\star}$ with $k^\star$ defined as:
 \begin{align}
    k^\star = \min
    \left\{ k \in I \,\middle|\, 
    \sum_{i \leq k} \alpha_i \geq  \sum_{i > k} \alpha_i~.
    \right\}
    \label{eq:algo}
 \end{align}
 \end{prop}
 \begin{proof}
 First, such $k^\star$ exists as it is the smallest member of a discrete, non-empty set. Indeed, since all $\alpha_i$ are nonnegative, the set contains at least $k=\vert I \vert$.
We now verify that $s^\star=1/\alpha_k^\star$ is a critical point of $f$. By definition of $k^\star$ we have that $ \sum_{i \leq k^\star} \alpha_i \geq  \sum_{i > k^\star} \alpha_i$ and  $\sum_{i < k^\star} \alpha_i <  \sum_{i \geq k^\star} \alpha_i$.
By combining these two inequalities, we have that
 \begin{align}
 -\sum_{i < k^\star} \alpha_i + \sum_{i > k^\star} \alpha_i &\in [-\alpha_{k^\star}, \alpha_{k^\star}]~.
 \label{eq:ineq}
 \end{align}
 Since $I$ orders the $\alpha_i$ in increasing order, we can write the  subgradient of $f$ at $s^\star$ under the following form:
 \begin{align}
    \partial_s f(s^\star) 
    &= 
    \sum_{i < k^\star} \partial_s \abs{s^\star \alpha_{i} -1}
    +
    \sum_{i > k^\star} \partial_s \abs{s^\star \alpha_{i} -1}
    +
    \partial_s \abs{s^\star \alpha_{k^\star} -1} \\
    &= 
    - 
    \sum_{i < k^\star} \alpha_{i}
    +
    \sum_{i > k^\star} \alpha_{i}
    +
     [-\alpha_{k^\star},\alpha_{k^\star}].
 \end{align}
 By using the inequality defined in \equaref{eq:ineq}, we have that $0 \in \partial_s f(s^\star)$ and hence $s^\star$ is a critical point of $f$. Since $f$ is convex, such $s^\star$ is also a global minimizer of $f$, \ie an optimal scaling.
  \end{proof}
This proposition gives us a fast algorithm to obtain an optimal scaling and hence a scale-free distortion:
compute the cumulative sum of the $\alpha_{k,l}$ sorted in ascending order until the equality in \eqref{eq:algo} is first verified at index $k^\star$. The resulting optimal scaling is then given by $1/\alpha_{k^\star}$. 
\subsection{Smooth Distortion}
The minimization problem with respect to $s$ defined in \equaref{eq:loss:distosmooth} can be solved in closed form:
\begin{align}
\label{eq:loss:minimizer}
s^\star &= {\sum \frac{d(\pi_{k},\pi_{l})}{D[k,l]}}\bigg/{\sum \frac{d(\pi_{k},\pi_{l})^2}{D[k,l]^2}} ~.
\end{align}
\subsection{Evolution of Optimal Scaling}
\label{sec:apx:optscale}
\begin{figure}[h!]
    \centering
    \includegraphics[width=\linewidth]{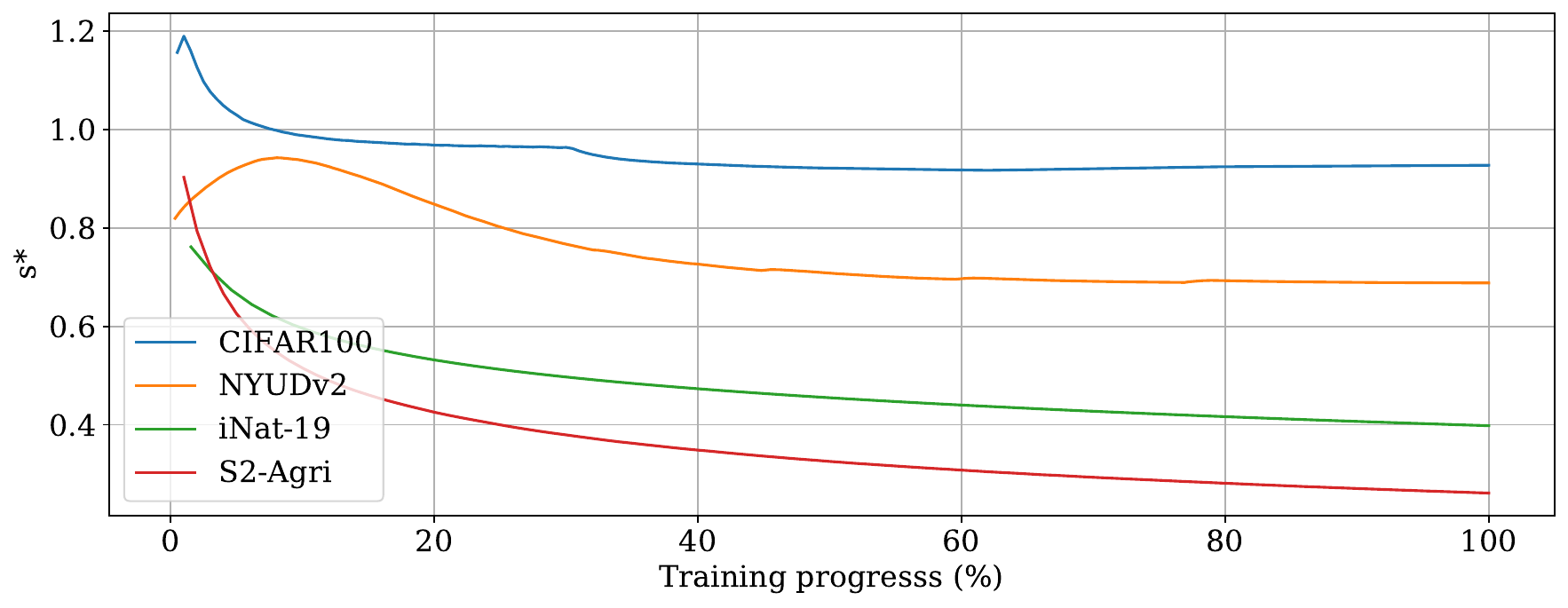}
    \caption{Evolution of the scaling factor $s^*$ in $\Ldisto$ along the training iterations of the four networks. We observe that $s^*$ consistently decreases to values smaller than $1$, which allow the prototypes to spread apart while respecting the fix distances defined by $D$.
    }
    \label{fig:scale}
\end{figure}
In \figref{fig:scale}, we represent the evolution of the scaling factor $s^*$ in $\Ldisto$ during training of our guided prototype method on the four datasets. Across all four models, $s^*$ presents a decreasing trend overall, which signifies that the average distance between prototypes increases. This is consistent with our analysis of prototypical networks: as the feature learning network and the prototypes are jointly learned, the samples' representations get closer to their true class' prototype. In doing so, they repel the other prototypes, which translate into an \emph{inflation} of the global scale of the problem. Our optimal scaling allows the prototypes' scale to expand accordingly. Without adaptive scaling, the data loss \eqref{eq:ldata_long} and regularizer \eqref{eq:loss:distosmooth} would conflict.

In all our experiments, this scale remained bounded and did not diverge. This can be explained by the fact that for each misclassification $k \rightarrow l$ of a sample $x_n$, the representation $f(x_n)$ is by definition closer to the erroneous prototype $\pi_l$ than of the true prototype $\pi_k$. The first term of $\mathcal{L}_\text{data}$ pushes the true prototype $\pi_k$ towards $f(x_n)$, and by transitivity---towards the erroneous prototype $\pi_l$. This phenomenon prevents prototypes from being pushed away from one another indefinitely.
However, if the prediction is too precise, \ie most samples are correctly classified, the prototypes may diverge. This setting, which we haven't yet encountered, may necessitate a regularization such as weight decay on the prototypes parameters.

Lastly, we remark that the asymptotic optimal scalings are different from one dataset to another. This can be explained foremost by differences in the depth and density of the class hierarchy of each dataset, as presented in \tabref{tab:nomenc}. As explained above, the inherent difficulty of the classification tasks also have an influence on the problem's scale. However, our parameter-free method is able to automatically find an optimal scaling.

\subsection{Inference}
As with other prototypical networks, we associate to a sample $n$ the class $k$ whose prototype $\pi_k$ is the closest to the representation $f(x_n)$ with respect to $d$, corresponding to the class of highest probability.
This process can be made efficient for a large number of classes $K$ and a high embedding dimension $m$ with a KD-tree data structure, which offers a query complexity of $O(\log(K))$ instead of $O(K\cdot m)$ for an exhaustive comparison. Hence, our method does not induce longer inference time than the cross-entropy for example, as the embedding function typically takes up the most time.

\subsection{Rank-based Guiding}
Following the ideas of  \citet{mettes2019hyperspherical}, we also experiment with a RankNet-inspired loss \citep{ranknet} which encourages the distances between prototypes to follow \emph{the same order} as the costs between their respective classes, without imposing a specific scaling: 
\begin{align}\label{eq:loss:rank}
    \Lrank(\pi) &= - \frac{1}{\vert\cT\vert }
\sum_{k,l,m \in \cT} 
\bar{\bf R}_{k,l,m} \cdot \text{log}(R_{k,l,m}) + (1 - \bar{\bf R}_{k,l,m}) \cdot \text{log}(1- R_{k,l,m}) 
~,
\end{align}
with $\cT=\{(k,l,m) \in  \cK^3\mid k\neq l, l\neq m, k\neq m\}$ the set of ordered triplet of $\cK$,  $\bar{\bf R}_{k,l,m}$ the hard ranking of the costs between $D_{k,l}$ and $D_{k,m}$, equal to $1$ if $D_{k,l} > D_{k,m} $ and $0$ otherwise, and $R_{k,l,m} = \text{sigmoid}(d(\pi_k,\pi_l) - d(\pi_k,\pi_m))$ the soft ranking between $d(\pi_k,\pi_l)$ and $d(\pi_k,\pi_m)$.
For efficiency reasons, we sample at each iteration only a $S$-sized subset of $\cT$. We use $S=10$ in our experiments.

\FloatBarrier

\section{Additional experimental details}
We give additional details on our experiments and some supplementary results in the following subsections. 

\subsection{Competing methods}
\paragraph{Hierarchical Cross-Entropy} (HXE)
\citet{bertinetto2020making} model the class structure with a hierarchical loss composed of the sum of the cross-entropies at each level of the class hierarchy. As suggested, a parameter $\alpha$ taken as $0.1$ defines exponentially decaying weights for higher levels.
%
\paragraph{Soft Labels} (Soft-labels)
\citet{bertinetto2020making} propose as second baseline in which the the one-hot target vectors are replaced by soft target vectors in the cross-entropy loss. These target vectors are defined as the softmin of the costs between all labels and the true label, with a temperature ${1}/{\beta}$ chosen as $0.1$, as recommended in \citet{bertinetto2020making}.
%
\paragraph{Earth Mover Distance regularization}(XE+EMD):
\citet{hou2016squared} propose to account for the relationships between classes with a regularization based on the squared earth mover distance. We use $D$ as the ground distance matrix between the probabilistic prediction $p$ and the true class $y$. This regularizer is added along the cross-entropy with a weight of $0.5$ and an offset $\mu$ of $3$.
\paragraph{Hierarchical Inference} (YOLO):
\citet{redmon2017yolo9000} propose to model the hierarchical structure between classes into a tree-shaped graphical model. First, the conditional probability that a sample belongs to a  class given its parent class is obtained with a softmax restricted to the class' co-hyponyms (\ie siblings). Then, the posterior probability of a leaf class is given by the product of the conditional probability of its ancestors. The loss is defined as the cross-entropy of the resulting probability of the leaf-classes. 
%
\paragraph{Hyperspherical Prototypes} (Hyperspherical-proto):
The method proposed by \citet{mettes2019hyperspherical} is closer to ours, as it relies on embedding class prototypes. They advocate to first position prototypes on the hypersphere using a rank-based loss (see \secref{sec:ablation}) 
combined with a prototype-separating term. 
They then use the squared cosine distance between the image embeddings and prototypes to train the embedding network. 
Note that in our re-implementation, we used the finite metric defined by $D$ instead of Word2Vec \citep{mikolov2013word2vec} embeddings to position prototypes. Lastly, we do not evaluate  on S2-Agri as the integration of the focal loss is non-trivial. 
%
\paragraph{Deep Mean Classifiers} (Deep-NCM):
\citet{guerriero2018deepNCM} present another prototype-based approach. Here, the prototypes are the cumulative mean of the embeddings of the classes' samples, updated at each iteration. The embedding network is supervised with $\Ldata$ with $d$ defined as the squared Euclidean norm.
%
\subsection{Numerical results}
\label{sec:numres}
The numerical values of the results shown in \figref{fig:results} are given in \tabref{tab:results}.

-------------------------------------------------------------
\begin{figure}[h!]
\centering
\begin{tabular}{cc}
	\subfigure[Porcupine $\leftrightarrow$ Shrew  $-80\%$, $D_a=4$ ]{
	\captionsetup{justification=centering}
		\centering
		\includegraphics[width=2.5cm, trim=2.5cm 2.5cm 2.5cm 2.5cm, clip]{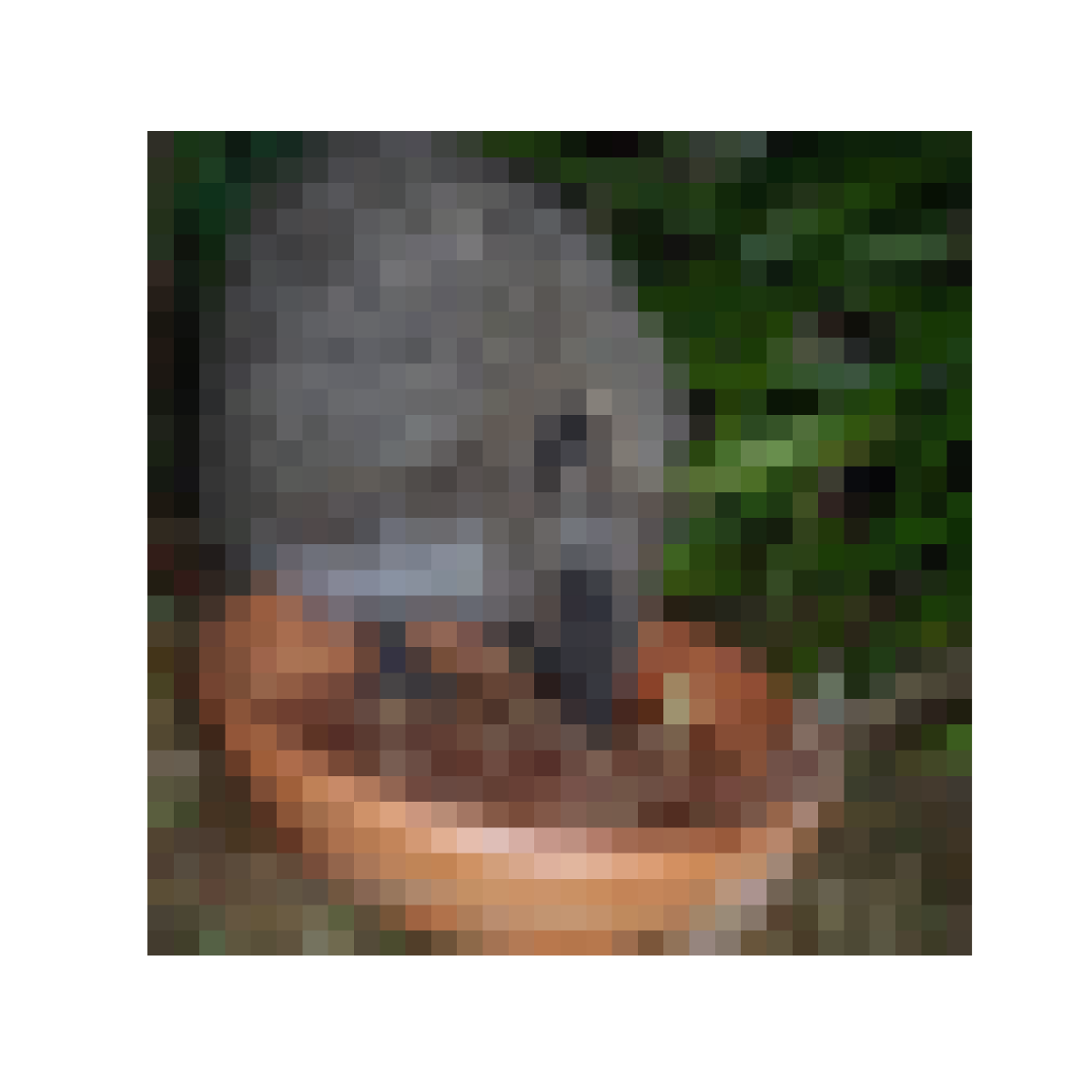}
		\includegraphics[width=2.5cm, trim=2.5cm 2.5cm 2.5cm 2.5cm, clip]{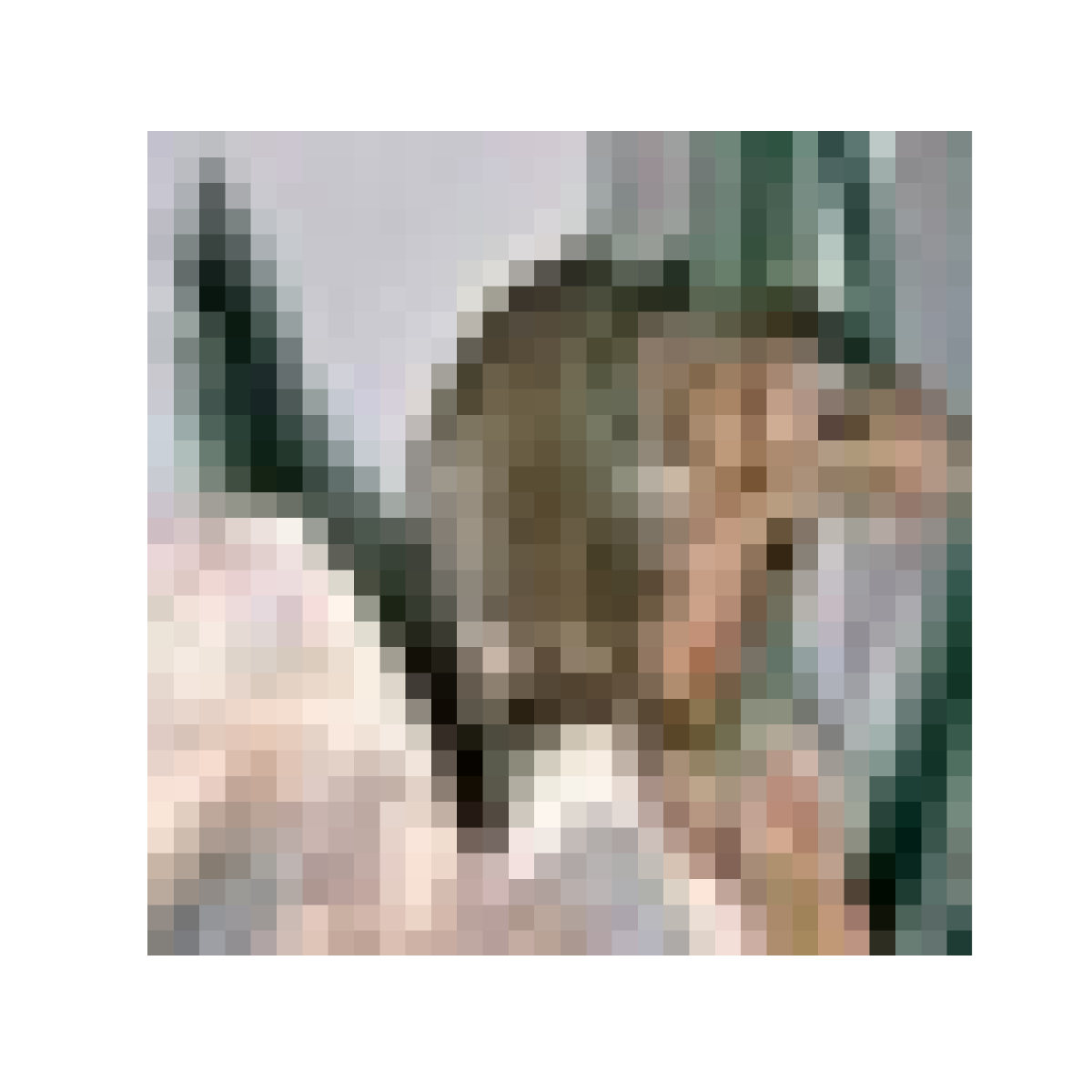}
	}
	&
	\subfigure[Caterpillar  $\leftrightarrow$  Lizard  $-64\%$, $D_b=4$]{
	\captionsetup{justification=centering}
		\centering
		\includegraphics[width=2.5cm, trim=2.5cm 2.5cm 2.5cm 2.5cm, clip]{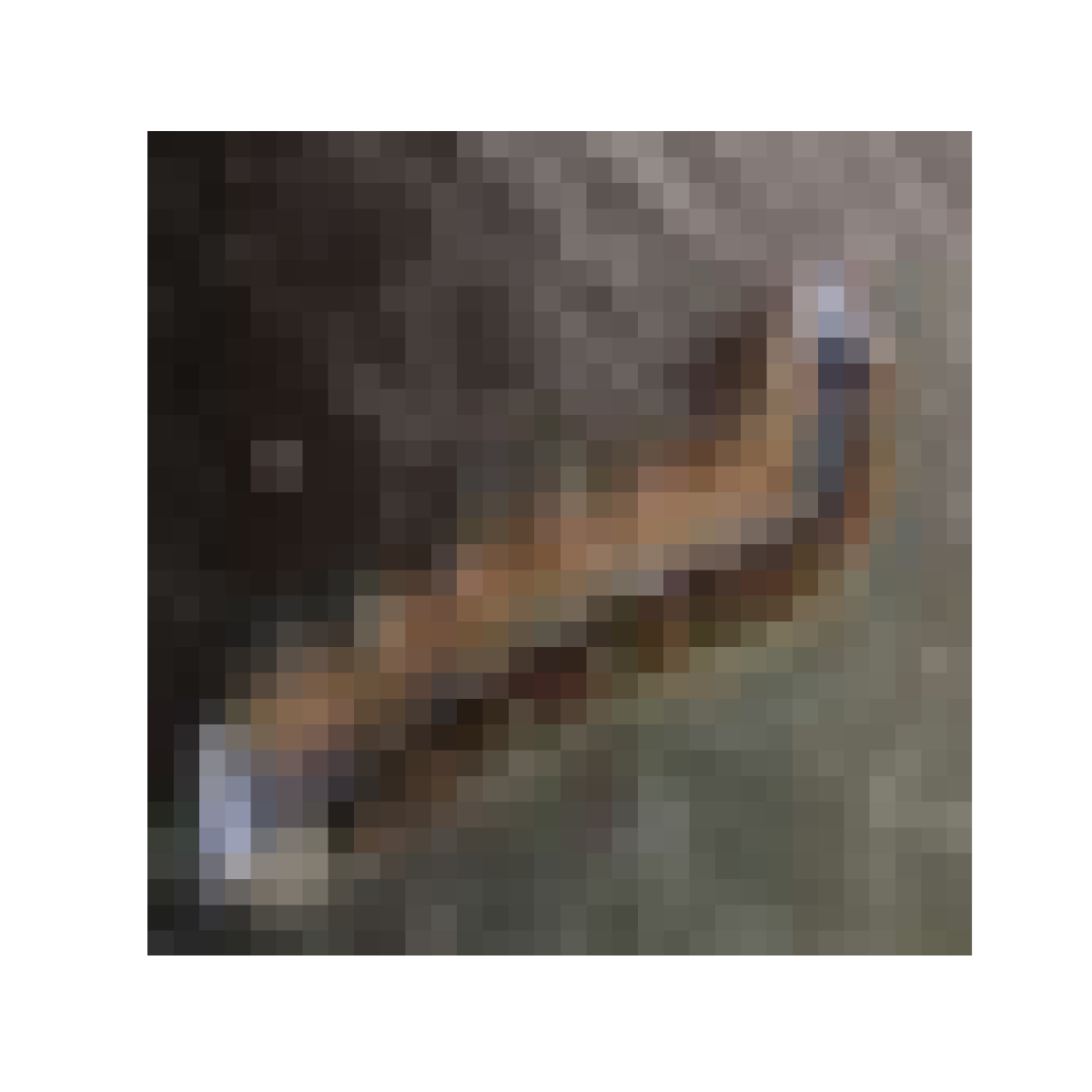}
		\includegraphics[width=2.5cm, trim=2.5cm 2.5cm 2.5cm 2.5cm, clip]{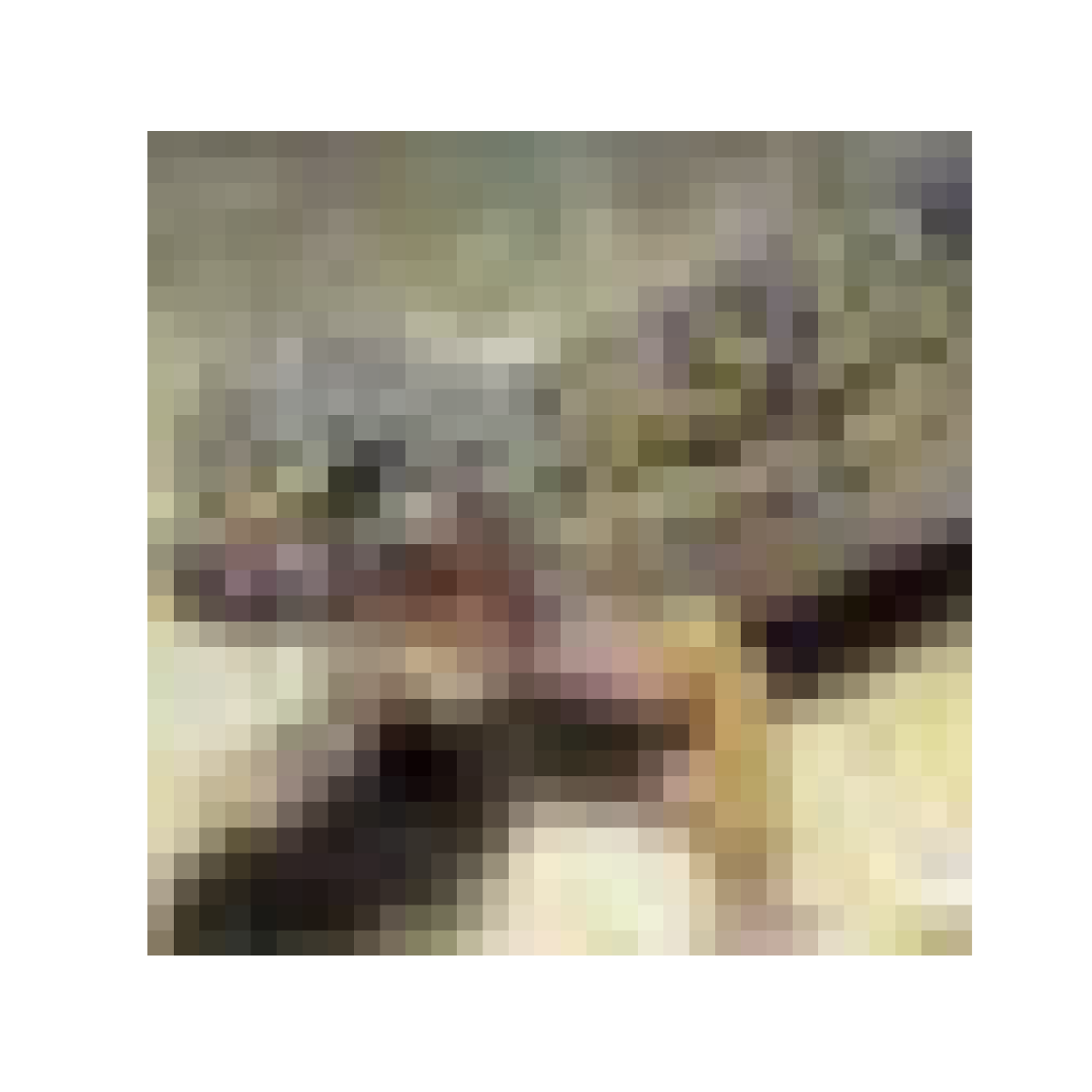}
	}
	\\
	\subfigure[Plate  $\leftrightarrow$  Clock  $-40\%$, $D_c=4$ ]{
	\captionsetup{justification=centering}
		\centering
		\includegraphics[width=2.5cm, trim=2.5cm 2.5cm 2.5cm 2.5cm, clip]{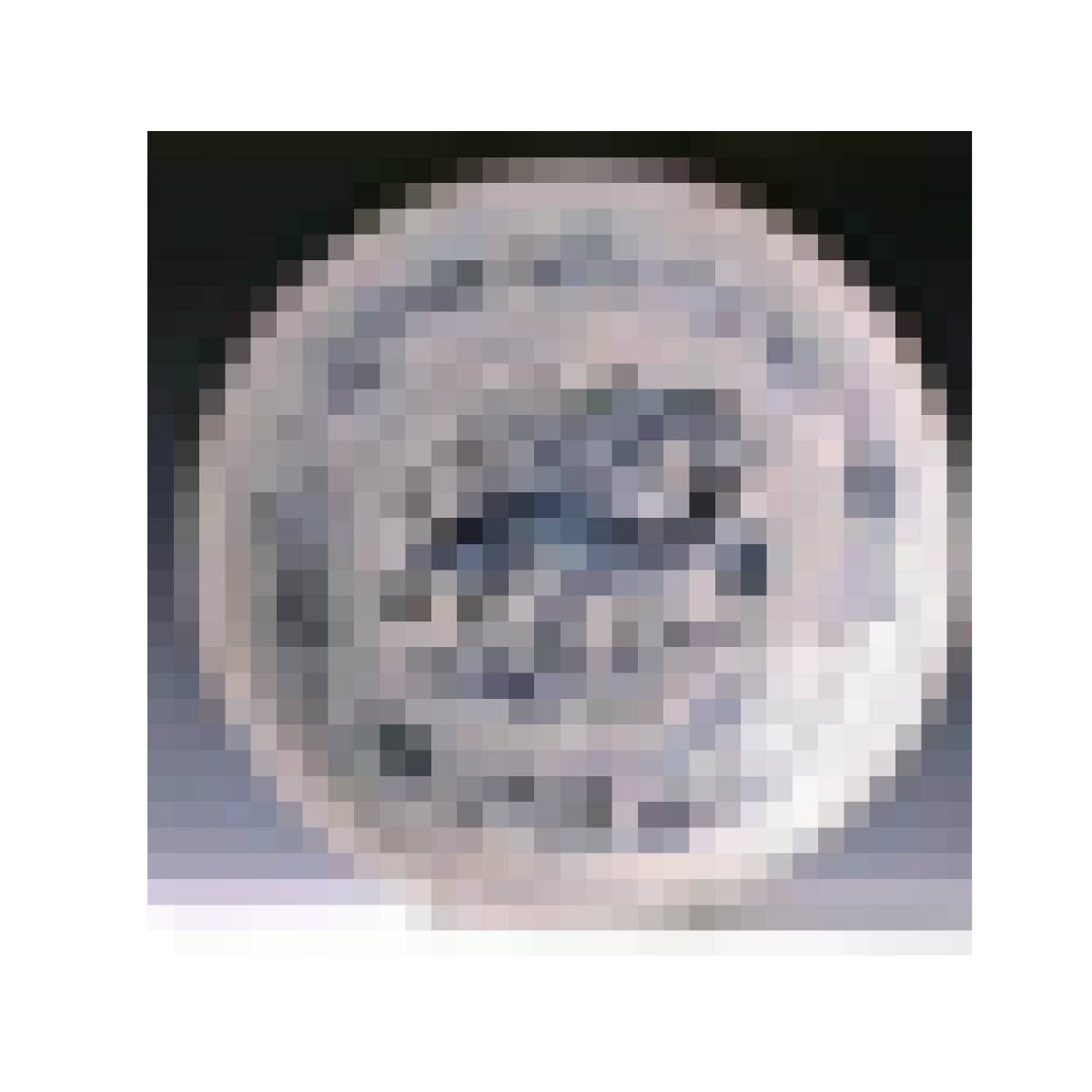}
		\includegraphics[width=2.5cm, trim=2.5cm 2.5cm 2.5cm 2.5cm, clip]{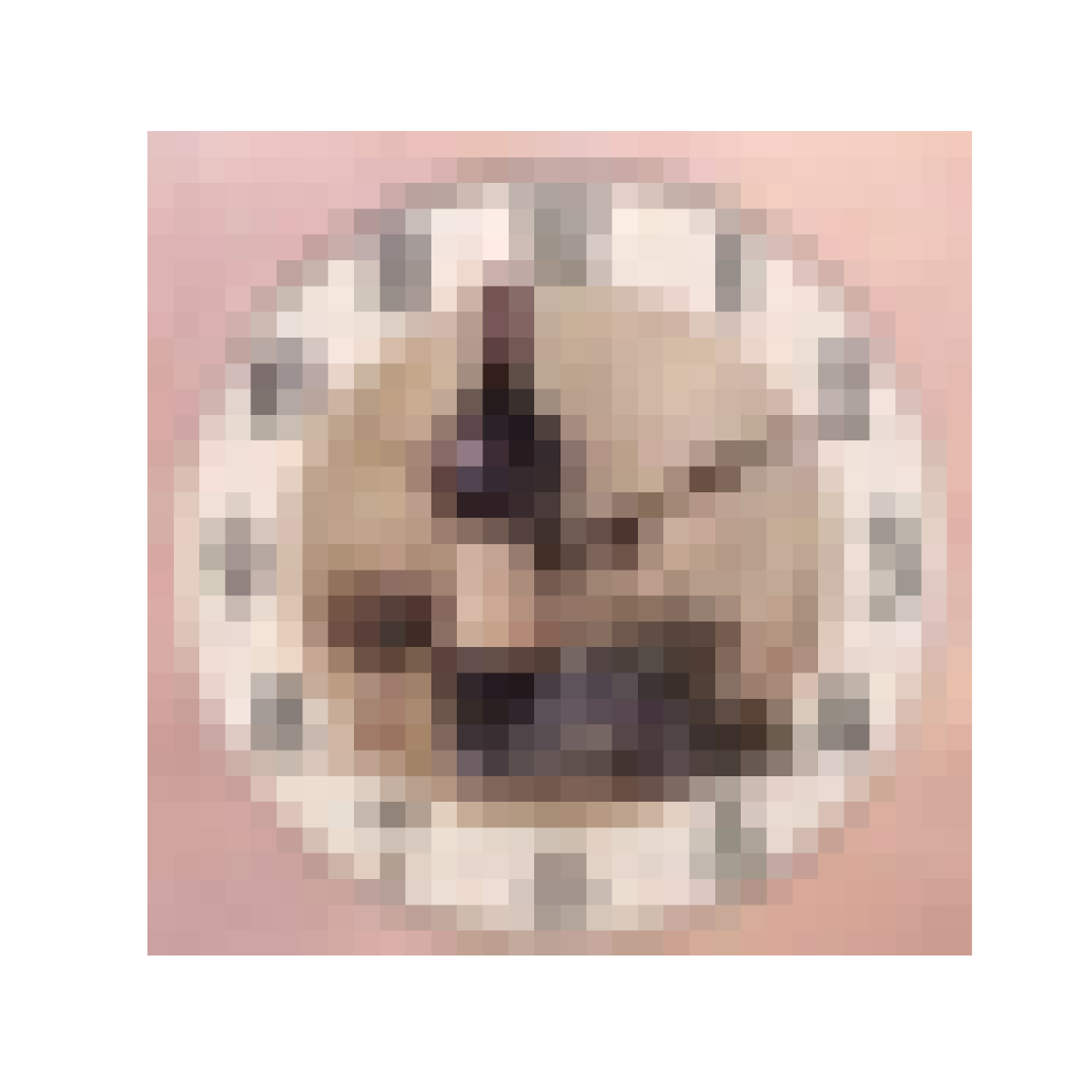}
	}
	&
	\subfigure[Streetcar  $\leftrightarrow$  Bus  $+58\%$, $D_d=4$]{
	\captionsetup{justification=centering}
		\centering
		\includegraphics[width=2.5cm, trim=2.5cm 2.5cm 2.5cm 2.5cm, clip]{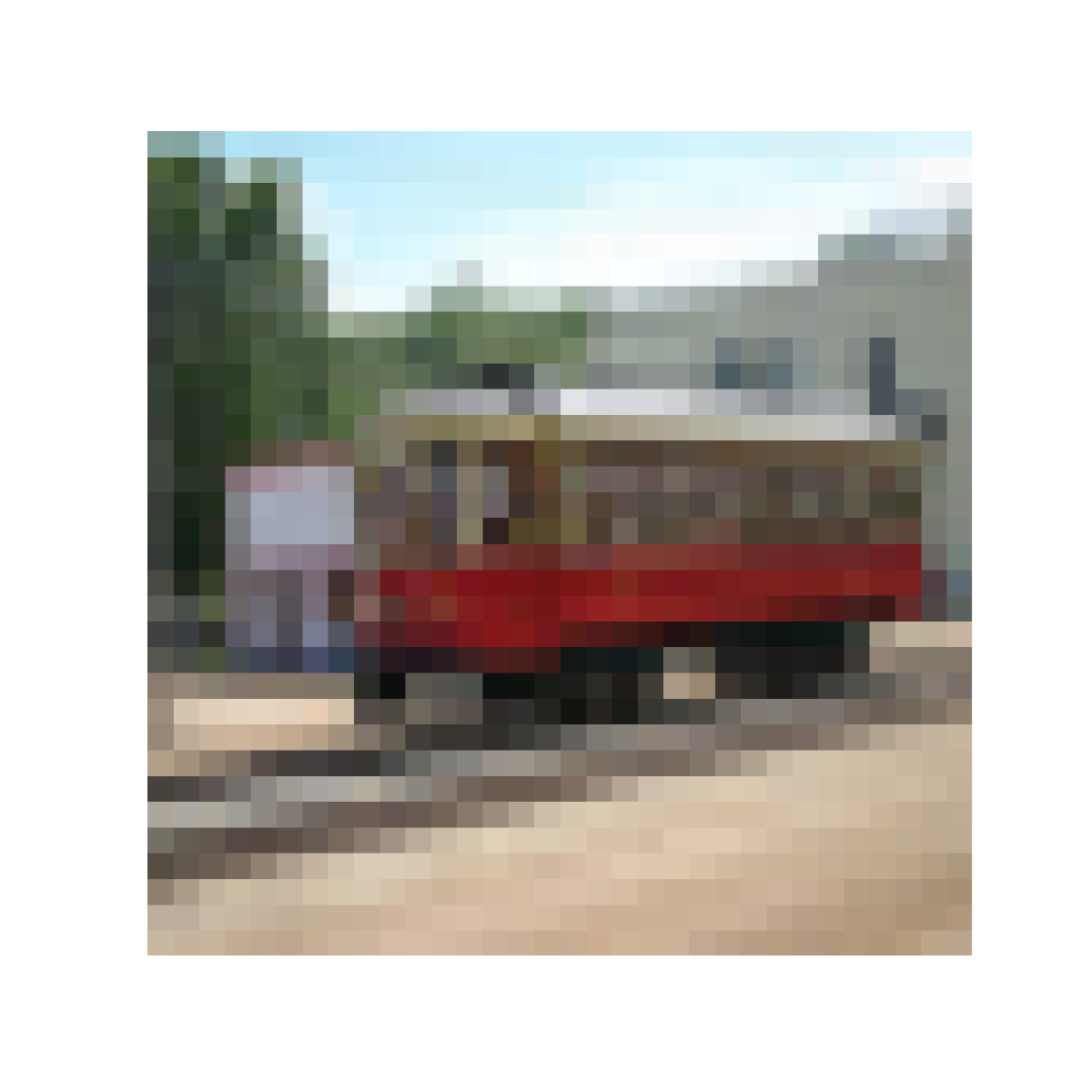}
		\includegraphics[width=2.5cm, trim=2.5cm 2.5cm 2.5cm 2.5cm, clip]{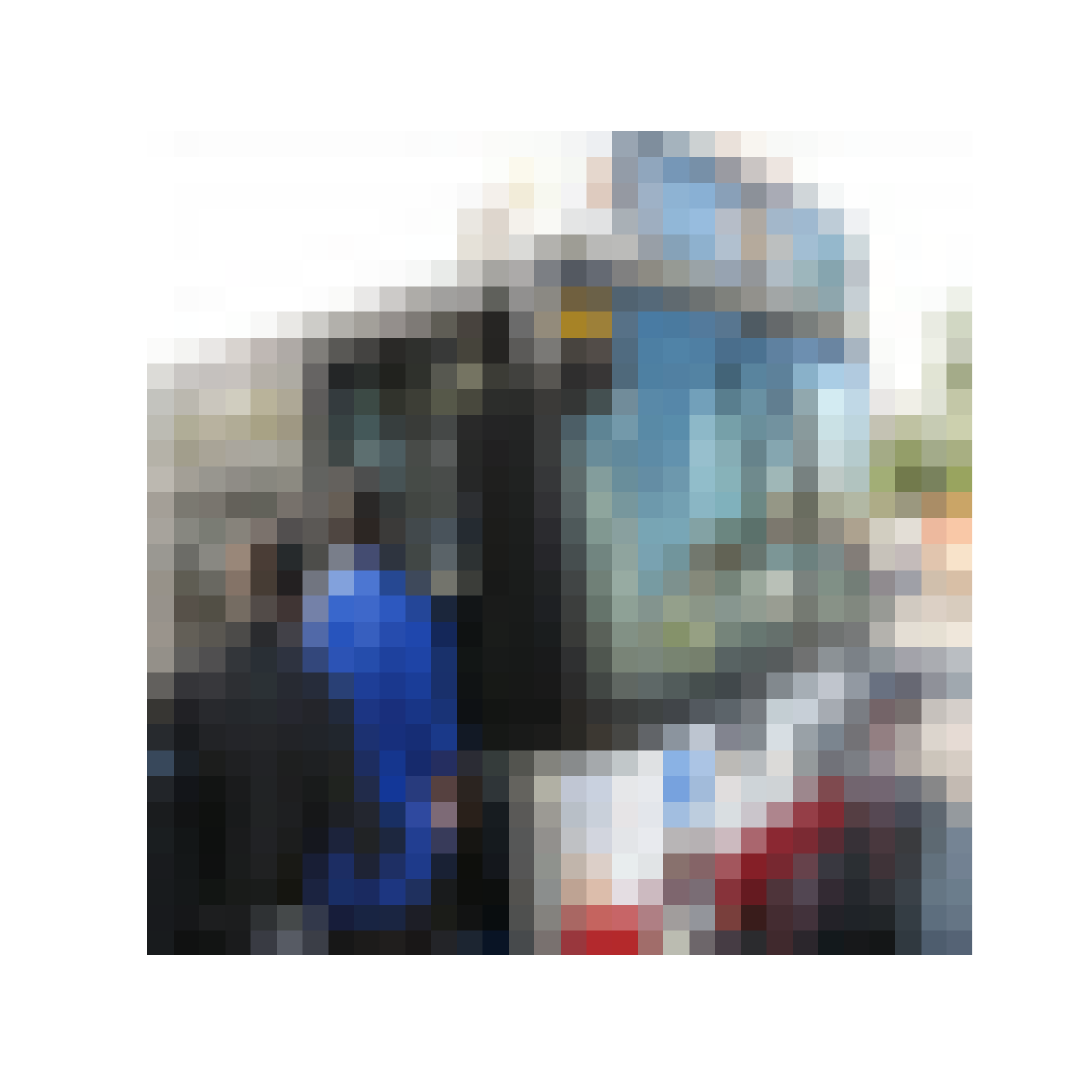}
	}\\
	\subfigure[Otter  $\leftrightarrow$  Seal  $+60\%$, $D_e=2$]{
	\captionsetup{justification=centering}
		\centering
		\includegraphics[width=2.5cm, trim=2.5cm 2.5cm 2.5cm 2.5cm, clip]{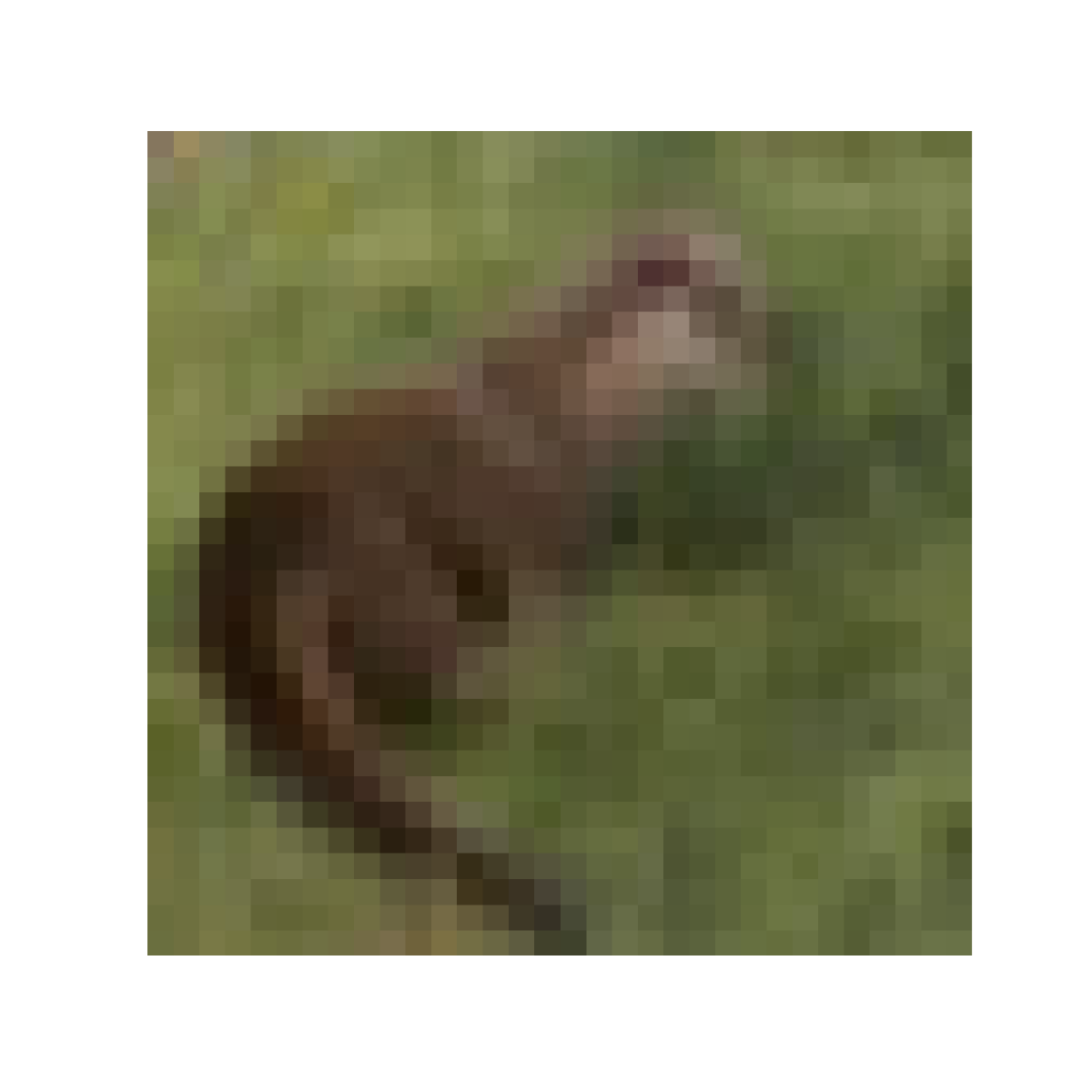}
		\includegraphics[width=2.5cm, trim=2.5cm 2.5cm 2.5cm 2.5cm, clip]{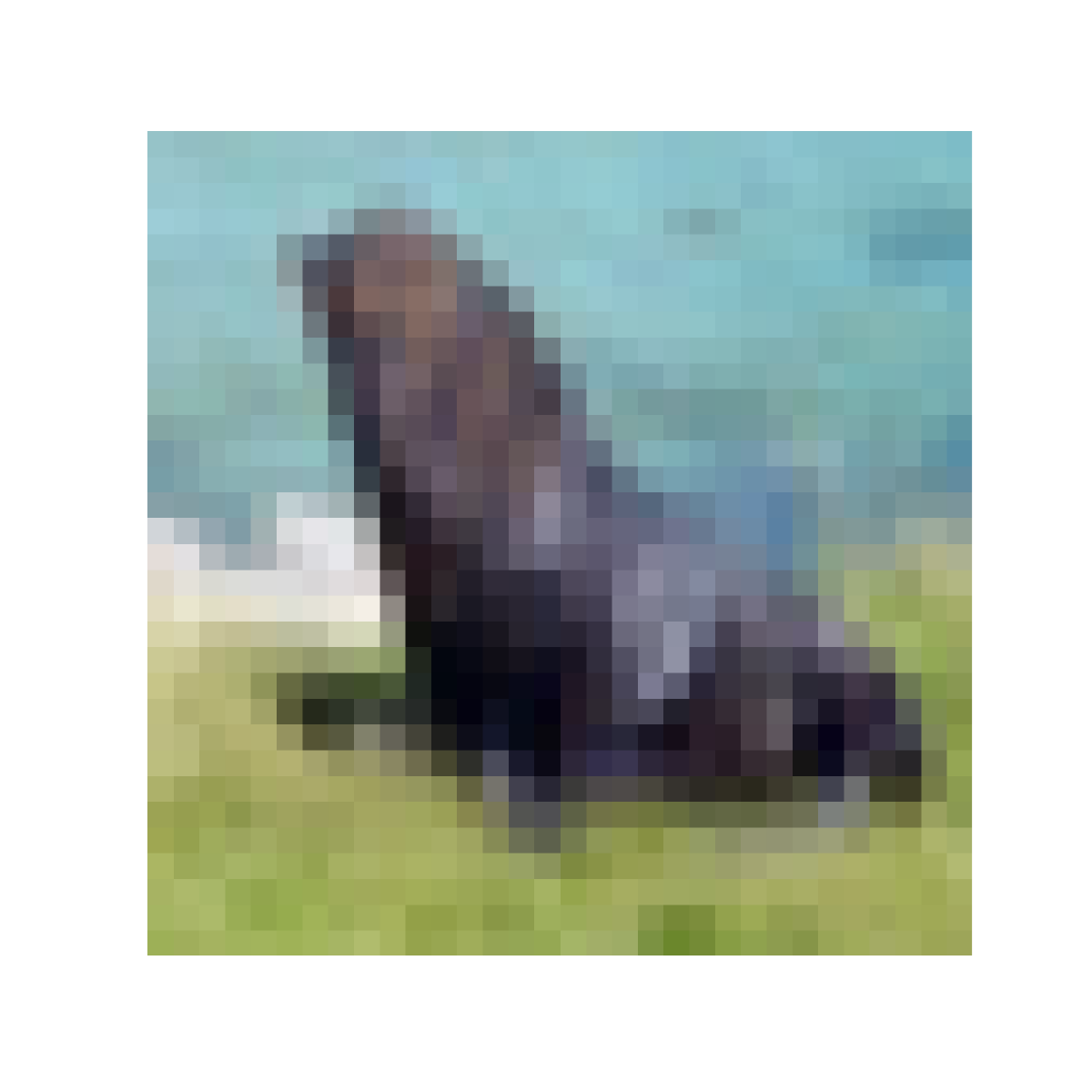}
	}&
	\subfigure[Boy  $\leftrightarrow$  Man  $+78\%$, $D_f=2$ ]{
	\captionsetup{justification=centering}
		\centering
		\includegraphics[width=2.5cm, trim=2.5cm 2.5cm 2.5cm 2.5cm, clip]{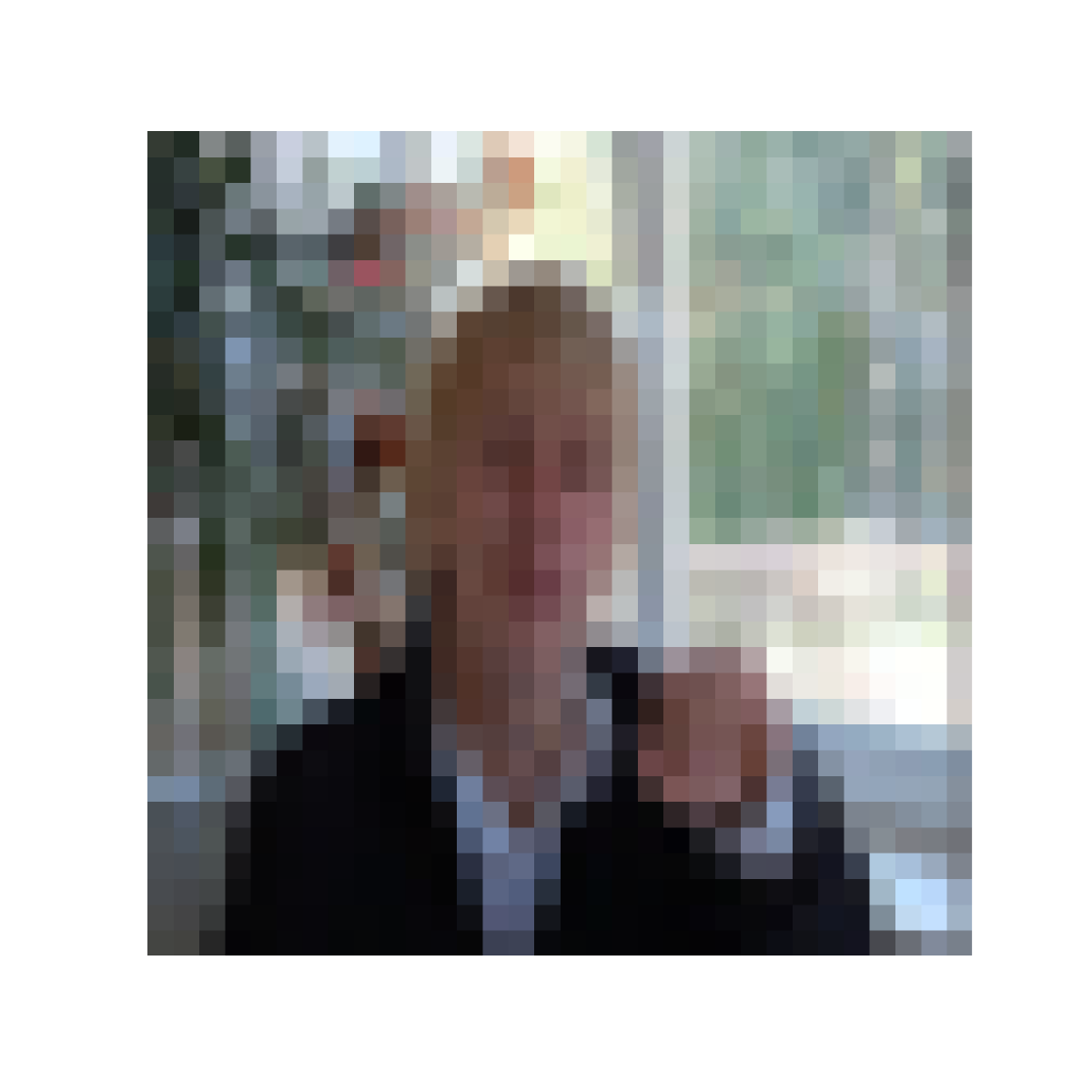}
		\includegraphics[width=2.5cm, trim=2.5cm 2.5cm 2.5cm 2.5cm, clip]{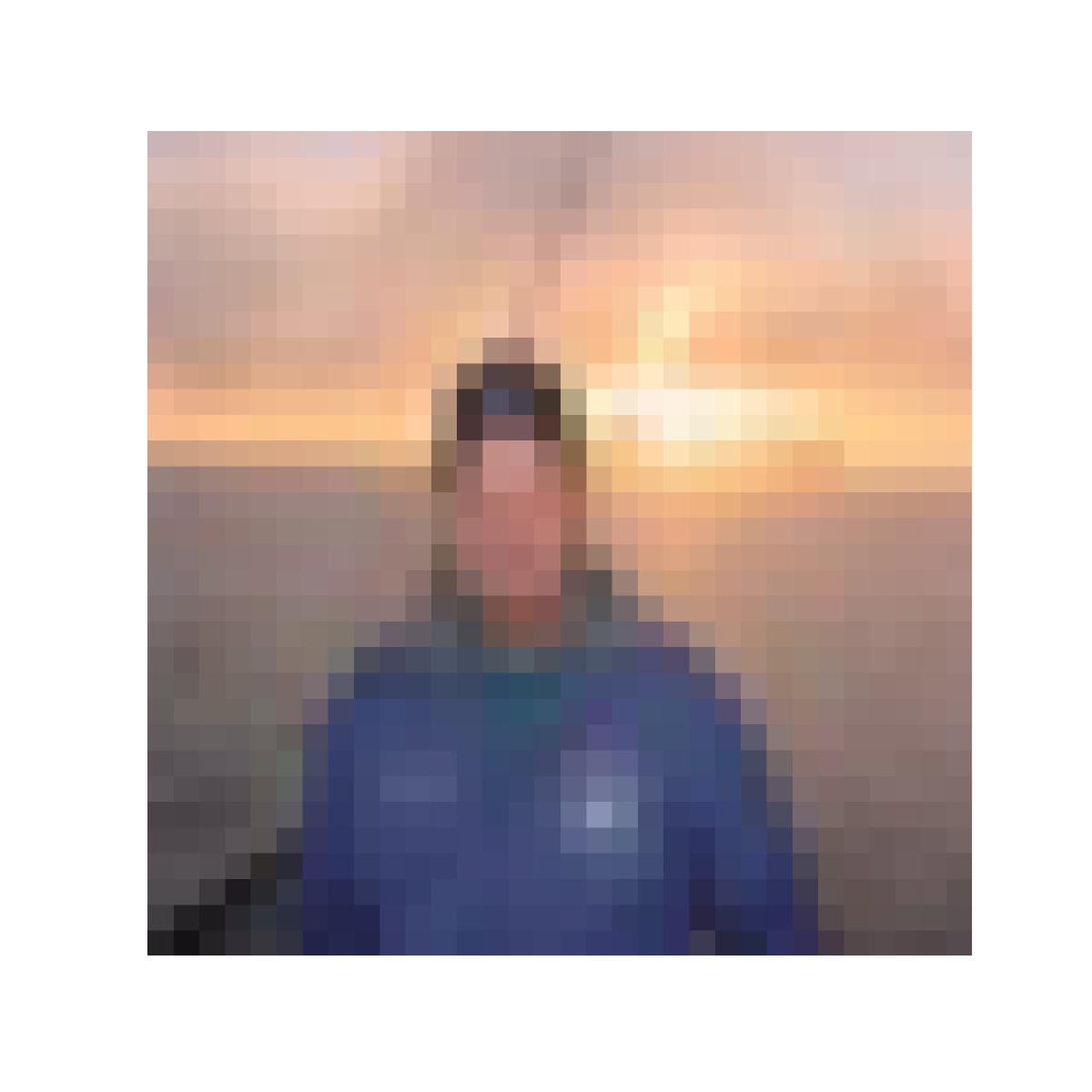}
	}
	\end{tabular}
	\caption{Best (a-c) and worse (d-f) improvements in terms of class  confusion provided by  Guided-proto compared to the cross-entropy baseline for CIFAR100, given in $\%$, along with their error cost. The metric guided regularization particularly helps decreasing the confusions between classes that are visually similar (\eg Plate and Clock) but are not direct siblings in the class hierarchy ($D=4$).
	Conversely, the regularization hinders performance for visually similar siblings classes (\eg Otter and Seal, $D=2$). } 
	\label{fig:pairs}
\end{figure}

\begin{table*}[h]
\caption[]{Error Rate (ER) in \% and Average Hierarchical Cost (AHC) on three datasets for our proposed methods (top) and the competing approaches (bottom). The values are computed with the median over $5$ runs for CIFAR100, the average over $5$ cross-validation folds for S2-Agri, and a single run for NYUDv2 and iNat-19. (HSP: Hyperspherical Prototypes, GP: Guided Prototypes).
}
\vspace{.2cm}
\centering
\ra{1.3}
\small{
\begin{tabular}{@{}lccccccccccc@{}}
& \multicolumn{2}{c}{CIFAR100} & \phantom{a}& \multicolumn{2}{c}{NYUDv2} & \phantom{a} & \multicolumn{2}{c}{S2-Agri}& \phantom{a} & \multicolumn{2}{c}{iNat-19}\\\cmidrule{2-3} \cmidrule{5-6} \cmidrule{8-9} \cmidrule{11-12}&       ER & AHC          & & ER         & AHC     &    & ER & AHC      &&ER &AHC  \\\toprule

Cross-Entropy 
&     24.2         &     1.160          &&        32.7   &      1.486    & &   19.4        &     0.699 && 40.9 & 1.993       \\\midrule
HXE  &     24.1  &     1.168      &&     32.4    &  1.456  & &   19.5  &      0.731 && 41.8 & 2.013            \\
Soft-label  &     23.5          &      \bf1.046        &&     32.4    &  \textbf{1.424}  & &   19.2  & 0.703 && 52.8 & 2.029                  \\
XE+EMD  &     24.5        &      1.196         &&  33.3      &    1.498 & &    19.0 &   0.687 && 40.1 & 1.893                \\
YOLO  
&   26.2   &       1.214        &&     \textbf{ 32.0}      &     \bf1.425       & &       19.1     &  \bf 0.685 && 42.0 & 1.942 \\
HSP 
&      29.4       &      1.472         &&     49.7        &     2.329       & &  - &     - && 42.4 & 2.027  \\
Deep-NCM   &     25.6        &      1.249         &&        33.5   &      1.498       &&    19.4   &    0.702 && 40.8 & 1.929  \\\midrule
Free-proto   &     23.8        &        1.091       &&       32.5      &       1.462    &&        19.1    &      0.691 && \bf 38.8 & \bf 1.728      \\ 
Fixed-proto    &       24.7      &      1.083      & &        33.1   &     1.462       &&      19.4   &  0.710 && 43.9 & 2.148   \\
GP-rank  &   \bf 23.3           &         \bf 1.056      & &        32.7         &     1.445            & &    19.1      &     0.691 && 39.3 & \bf1.718      \\
GP-disto     &     23.6         &        \bf 1.052      & &       32.5     &     1.440        &&     \bf 18.9      &    \bf  0.685 && \bf 38.9 & \bf1.721      \\
\bottomrule
\end{tabular}
}
\label{tab:results}
\end{table*}
\subsection{Ablation Studies}

%

\begin{table*}[t]
\caption{Influence of the choice of scaling in $\Ldisto$,  metric guiding regularizer, guiding scheme, and distance function $d$ on the performance of Guided-proto on the four datasets. For $d$, we compare the performance of the Euclidean norm, the pseudo-Huberized Euclidean norm, and the square Euclidean norm.}
\label{tab:ablation}
\vspace{.2cm}
\centering
\ra{1.3}
\footnotesize{
\begin{tabular}{@{}lccccccccccc@{}}
& \multicolumn{2}{c}{CIFAR100} & \phantom{}& \multicolumn{2}{c}{NYUDv2} & \phantom{} & \multicolumn{2}{c}{S2-Agri}& \phantom{} & \multicolumn{2}{c}{iNat-19}\\\cmidrule{2-3} \cmidrule{5-6} \cmidrule{8-9}\cmidrule{11-12}&       ER & AHC          & & ER         & AHC     &    & ER & AHC&    & ER & AHC        \\\toprule
\textbf{Guided-proto }  & 23.6  & \textbf{1.052 }    & & 32.5  & 1.440       && \textbf{18.9} & \textbf{0.685}     && \textbf{38.9} & 1.721  \\\cmidrule{2-3} \cmidrule{5-6} \cmidrule{8-9}\cmidrule{11-12}
Fixed-scale                                   & +0.1   & +0.003    & & 0.0        & 0.000       && +0.2 & \textbf{+0.001}     && +0.9 & 0.000  \\
Fixed-proto                              & +1.1   & +0.031    & & +0.6   & +0.013  && +0.5 & +0.025    && +5.0 & +0.427 \\
Rank-based guiding                                & \textbf{-0.3}    & +0.004    & & +0.2   & +0.005  && +0.2 & +0.006    && +0.4 & \textbf{-0.003}  \\

Squared Norm                 &     +1.0     &     +0.118       & &    0.0     & +0.005  && +0.6  & +0.022        && +2.2 & +0.233 \\ \bottomrule
\end{tabular}
}

\end{table*}

\paragraph{Choice of distance :} In \tabref{tab:ablation}, we report the performance of the Guided-proto model on the four datasets when replacing the Euclidean norm  with the squared Euclidean norm. Across our experiments, the squared-norm based model yields a worse performance. This is a notable result as it is the distance commonly used in most prototypical networks \citep{snell2017prototypical, guerriero2018deepNCM}. 
\paragraph{Rank-based Regularization:} \citet{mettes2019hyperspherical} use a rank-based loss \citep{ranknet} to encourage prototype mappings whose pairwise distance follows the same order as an external qualification of errors $D$.
We argue that our formulation of $\Ldisto$ provides a stronger supervision than only considering the order of distances, and allows the prototypes to find a more profitable arrangement in the embedding space. In \tabref{tab:ablation}, we observe that replacing our distortion-based loss by a rank-based one results in a slight decrease of overall performance.

\begin{table}[t!]
    \caption{
    Robustness assessment of guided prototypes on CIFAR100 (left) and S2-Agri (right). The top line is our chosen hyper-parameter configuration. }
    \vspace{.2cm}
    \label{tab:ablation:hyper}
    \centering
    \begin{tabular}{@{}lccccc@{} }
& \multicolumn{2}{c}{CIFAR100} & \phantom{abc}&  \multicolumn{2}{c}{S2-Agri}\\\cmidrule{2-3} \cmidrule{5-6}
& ER & AHC && ER & AHC \\ \toprule
    Guided-proto &  \multirow{2}{*}{23.6}  &   \multirow{2}{*}{\bf1.052}   &&\multirow{2}{*}{\bf18.9}  & \multirow{2}{*}{\bf0.685} \\
    $\lambda=1$, hidden proto,  &    &  & &    & \\\midrule
    $\lambda = 0.5$             & \bf -0.2  & +0.015     & &       +0.5  & +0.019\\
    $\lambda = 2$               &    +0.3     & +0.013     & &     +0.2 & +0.010\\
    $\lambda = 3$               &    +0.1    &   +0.004     & &    +0.1    & +0.010\\
    leaf proto only             &   +0.2    &  +0.015                  & &       +0.3 & +0.011  \\ \bottomrule
    \end{tabular}
\end{table}
\paragraph{Robustness:}
As shown in \tabref{tab:ablation:hyper}, our presented method has  low sensitivity with respect to regularization strength: models trained with $\lambda$ ranging from $0.5$ to $3$ yield sensibly equivalent performances. Choosing $\lambda = 1$ seems to be the best configuration in terms of AHC.
%
%
%
\paragraph{Hidden prototypes:} In cases where the cost matrix $D$ is derived from a tree-shaped class hierarchy, it is possible to also learn prototypes for the internal nodes of this tree, corresponding to super-classes of leaf-level labels. These prototypes do not appear in $\Ldata$, but can be used in the prototype penalization to instill more structure into the embedding space. In \tabref{tab:ablation:hyper}, line \emph{leaf-proto}, we note a small but consistent improvement in terms of AHC resulting in associating prototypes for classes corresponding to the internal-nodes of the tree hierarchy as well.
\subsection{Illustration of Results}

In \figref{fig:pairs} and \figref{fig:confusion}, we illustrate that our model particularly improves the classification rates of classes with high visual similarity and comparatively large error costs.  

\section{Additional Implementation Details}
\paragraph{CIFAR100} ResNet-18 is trained on CIFAR100 using SGD with initial learning rate $l_r=10^{-1}$,  momentum set to $0.9$ and weight decay $w_d = 5\cdot 10^{-4}$. The network is trained for $200$ epochs in batches of size $128$, and the learning rate is divided by $5$ at epochs $60$, $120$, and $160$. The model is evaluated using its weights of the last epoch of training, and the results reported in the paper are median values over $5$ runs.
\paragraph{NYUDv2} We train FuseNet on NYUDv2 using SGD with momentum set to $0.9$. The learning rate is set initially to $10^{-3}$ and multiplied at each epoch by a factor that exponentially decreases from $1$ to $0.9$. The network is trained for $300$ epochs in batches of $4$ with weight decay set to $5 \cdot 10^{-3}$. We report the performance of the best-of-five last testing epochs.
\paragraph{S2-Agri} We train PSE+TAE on S2-Agri using Adam with $l_r=10^{-3}$, $\beta=(0.9 ; 0.999)$ and no weight decay. The dataset is randomly separated in five splits. For each of the five folds, 3 splits are used as training data on which the network is trained in batches of $128$ samples for $100$ epochs. The best epoch is selected based on its performance on the validation set, and we use the last split to measure the final performance of the model. We report the average performance over the five folds.
\paragraph{iNaturalist-19} Given the complexity of the dataset, we follow \cite{bertinetto2020making} and use a ResNet-18 pre-trained on ImageNet. The network is trained for $65$ epochs in batches of $64$ epochs using Adam with $l_r=10^{-4}$, $\beta=(0.9 ; 0.999)$ and no weight decay. The best epoch is selected based on the performance on the validation set, and we report the performance on the held-out test set.
\section{Hierarchies used in Experiments}
We present here the hierarchy used in the numerical experiments to derive the cost matrix. We define the cost between two classes as the length of the shortest path in the proposed tree-shape hierarchy. The hierarchy of CIFAR100 is presented in \figref{fig:hier:cifar}, NYUDv2 in \figref{fig:hier:nyu}, S2-Agri in \figref{fig:hier:agri}, and iNat-19 in \figref{fig:hier:inat}.

For S2-Agri, we built the hierarchy by combining the two levels available in the dataset S2 of Garnot \etal withwith the fine-grained description of the agricutltural parcel classes on the French Payment Agency's website (in French):

\url{https://www1.telepac.agriculture.gouv.fr/telepac/pdf/tas/2017/Dossier-PAC-2017_notice_cultures-precisions.pdf}.

Note that for S2-Agri, following \cite{sainte2019satellite} we have removed all classes that had less than $100$ samples among the almost $200\,000$ parcels to limit the imbalance of the dataset. 

\begin{figure}
    \centering
    \includegraphics[width=\linewidth, trim=5.5cm 6cm 5.5cm 6cm, clip]{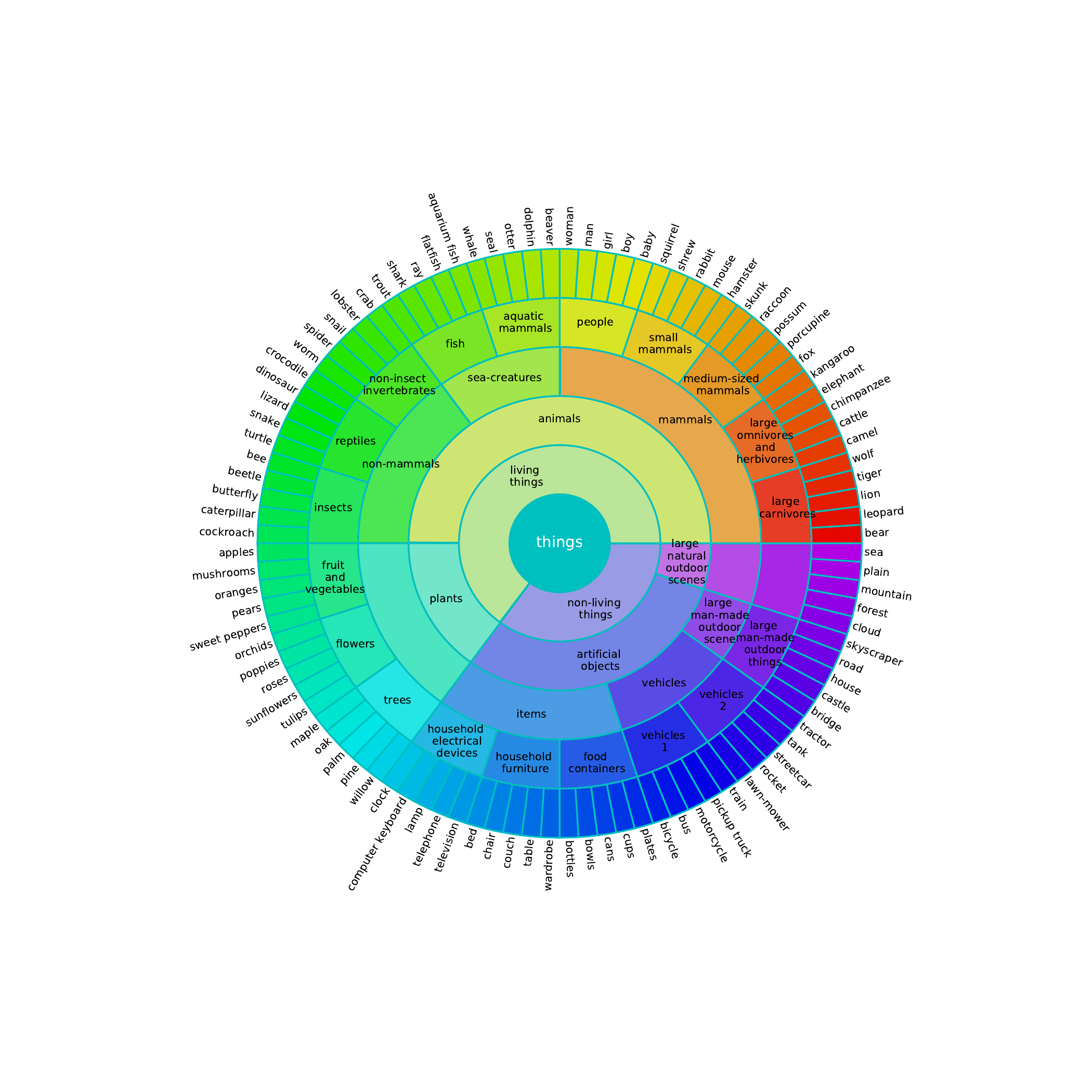}
    \caption{Class hierarchy for CIFAR100. The arcs at different radii represent the different classes of each level of the hierarchy. Unlabelled arcs share the same name as their parent class.}
    \label{fig:hier:cifar}
\end{figure}
\begin{figure}
    \centering
    \includegraphics[width=\linewidth, trim=5cm 6cm 5cm 6cm, clip]{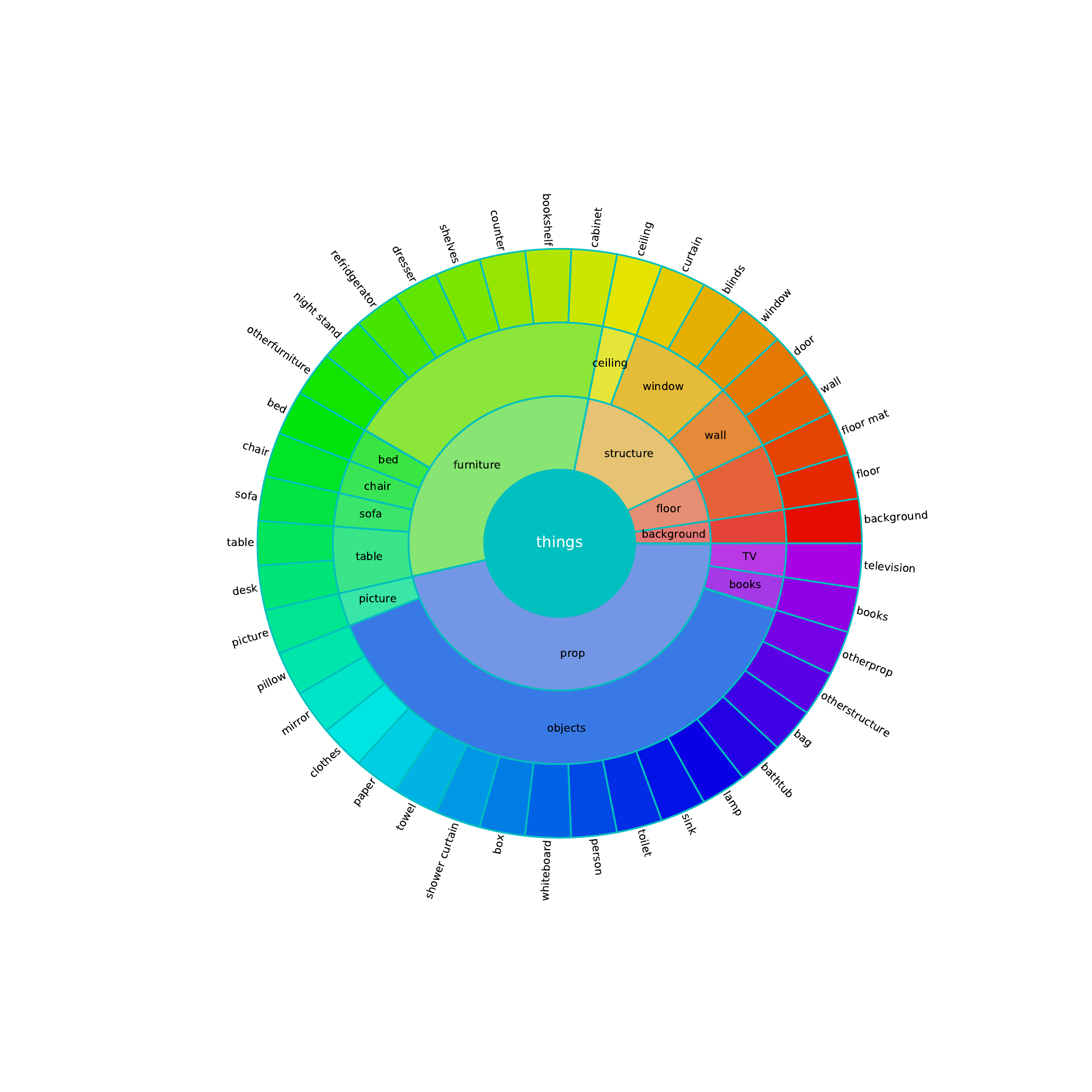}
    \caption{Class hierarchy for NYUv2}
    \label{fig:hier:nyu}
\end{figure}

\begin{figure}
    \centering
        \includegraphics[width=\linewidth, trim=4.5cm 4.5cm 3.8cm 3.8cm, clip]{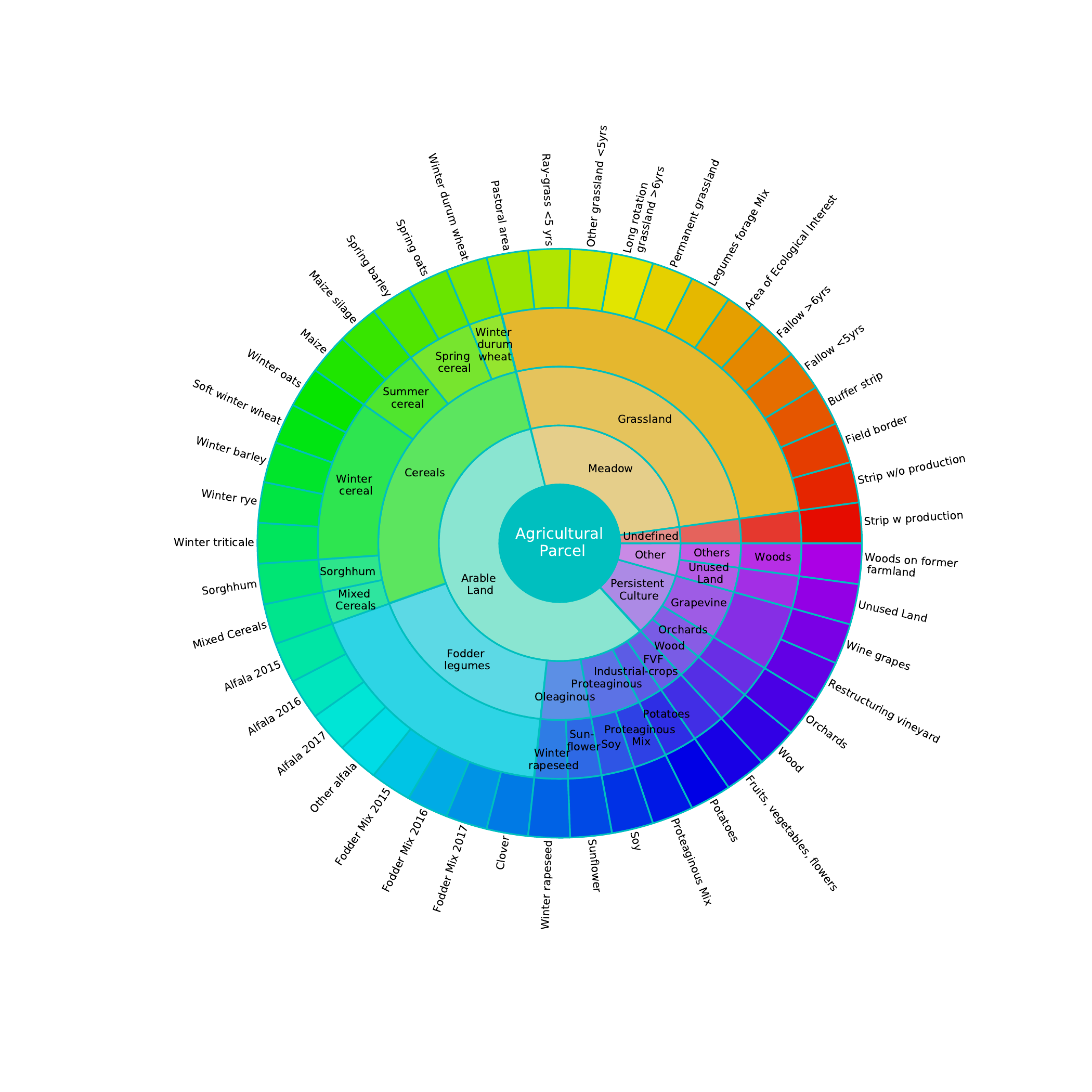}
    \caption{Class hierarchy for S2-Agri}
    \label{fig:hier:agri}
\end{figure}
\begin{figure}
    \centering
        \includegraphics[width=\linewidth, trim=6cm 6cm 6cm 6cm, clip]{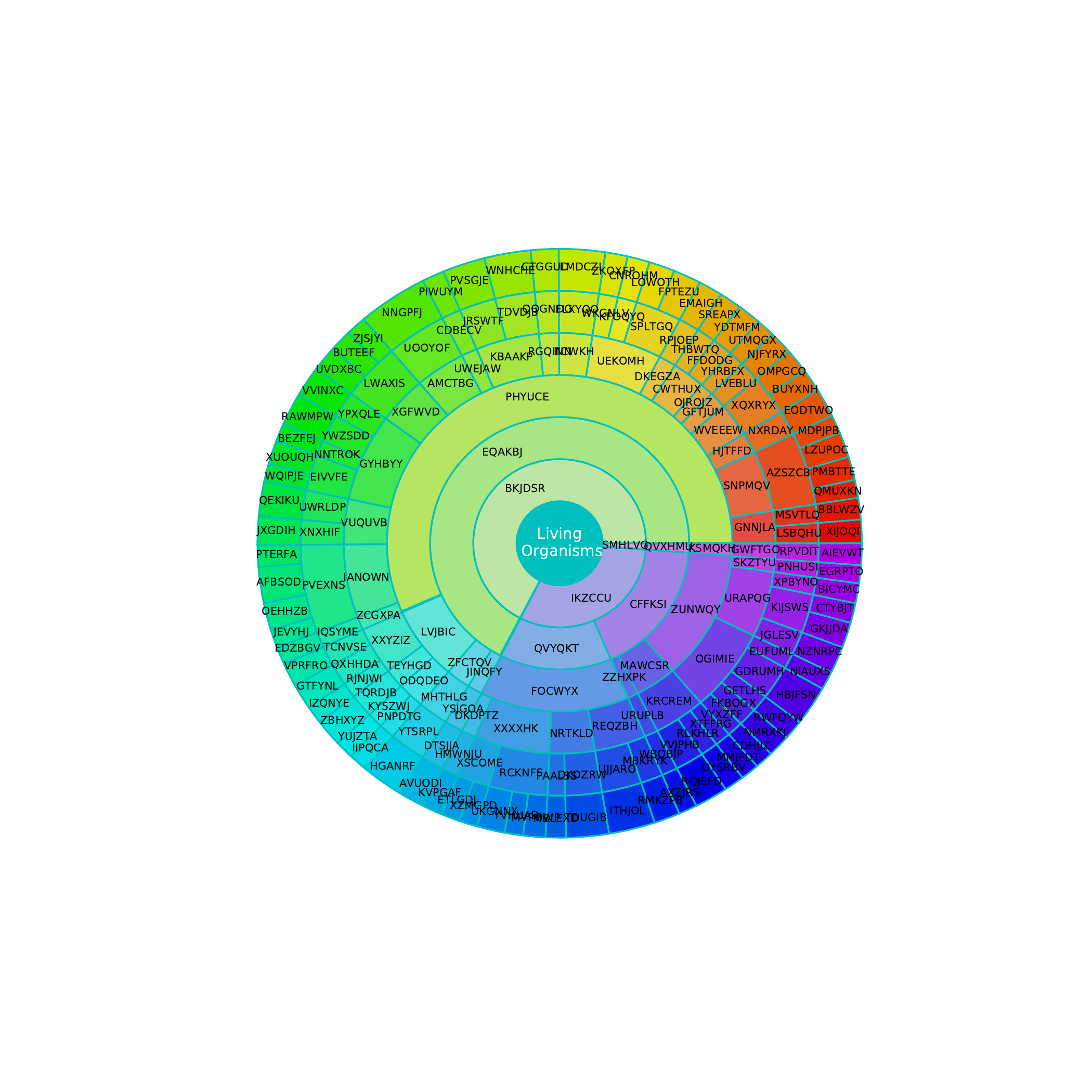}
    \caption{Class hierarchy for iNat-19, only the first $6$ levels of the hierarchy are represented. At the time of writing, only the classes' obfuscated names were publicly available}
    \label{fig:hier:inat}
\end{figure}}{}
\end{document}